\documentclass[twoside,11pt]{article}

\usepackage{natbib}
\usepackage{blindtext}
\usepackage[normalem]{ulem}
\usepackage[abbrvbib, preprint]{jmlr2e}

\usepackage{bookmark}
\usepackage{mathtools}
\usepackage{verbatim}   
\usepackage{ifthen}
\usepackage{bm}
\usepackage{lipsum}
\usepackage{hyperref}
\usepackage{amsmath, amssymb, enumerate, dsfont, comment, thmtools}
\usepackage[utf8]{inputenc}
\usepackage{csquotes}
\usepackage{tikz-cd}
\usepackage{tikz}
\usepackage{float}

\newcounter{assum}
\newenvironment{assum}{\par\noindent\refstepcounter{assum}\textbf{Assumption (A\arabic{assum})} \itshape}{\par}

\makeatletter
\@ifundefined{showdtodos}{
	\excludecomment{todo}
}{
	
}

\definecolor{darkgreen}{rgb}{0.0196, 0.3412, 0.1647}

\makeatletter
\@ifundefined{showdocu}{
	\excludecomment{docu}
}{
	\newenvironment{docu}{\par\color{darkgreen}}{\par}
}
\makeatother

\newcommand*{\Cov}{\operatorname{Cov}}

\newcommand{\tiid}{i.i.d.}

\newcommand{\innp}[2]{\langle #1, #2\rangle}

\newcommand{\set}[1]{\{#1\}}

\newcommand{\E}{\mathbb E}

\newcommand{\floor}[1]{\left\lfloor #1 \right\rfloor}

\newcommand{\der}{\partial}

\newcommand*{\cC}{\mathcal{C}}

\newcommand*{\cF}{\mathcal{F}}

\newcommand*{\cI}{\mathcal{I}}

\newcommand*{\cN}{\mathcal{N}}

\newcommand*{\cP}{\mathcal{P}}

\newcommand*{\cR}{\mathcal{R}}
\newcommand*{\cS}{\mathcal{S}}

\newcommand*{\cZ}{\mathcal{Z}}

\newcommand*{\N}{\mathbb{N}}
\newcommand*{\R}{\mathbb{R}}

\renewcommand*{\P}{\mathbb{P}}

\newcommand*{\om}{\omega}
\newcommand*{\Om}{\Omega}
\newcommand*{\si}{\sigma}
\newcommand*{\Si}{\Sigma}
\newcommand*{\al}{\alpha}
\newcommand*{\be}{\beta}
\newcommand*{\thet}{\theta}

\newcommand*{\delt}{\delta}

\newcommand*{\ph}{\varphi}
\newcommand*{\Ph}{\Phi}
\newcommand*{\la}{\lambda}
\newcommand*{\ka}{\kappa}
\newcommand*{\ga}{\gamma}

\newcommand*{\ep}{\varepsilon}

\newcommand{\Lip}{\operatorname{Lip}}

\newcommand*{\blnk}{\cdot}

\newcommand{\tIto}{Itô}

\newcommand{\tIe}{that is}

\newcommand*{\nrm}[2]{\|#1\|_{#2}}

\newcommand*{\Var}{\operatorname{Var}}

\newcommand*{\id}[1]{\operatorname{id}_{#1}}

\newcommand*{\tSDE}{stochastic differential equation}
\newcommand*{\tYDE}{Young differential equation}

\newcommand*{\mat}[1]{\begin{pmatrix}#1\end{pmatrix}}

\newcommand{\bx}{\bm{x}}
\newcommand{\by}{\bm{y}}
\newcommand{\bz}{\bm{z}}

\newcommand{\Sym}[1]{\cS_{#1}}

\newcommand{\pvar}[2]{\nrm{#1}{#2\operatorname{-var}}}

\newcommand{\specnrm}[1]{\nrm{#1}{\operatorname{op}}}

\newcommand{\hoelHZ}[1]{\dot \cC^{#1}}

\newcommand{\dC}[1]{\cC^{#1}}

\newcommand{\CHoelLoc}[1]{\cC^{0,#1}_{\operatorname{loc}}}

\newcommand{\LpLoc}[1]{L^{#1}_{\operatorname{loc}}}

\newcommand{\modu}{\operatorname{mod}}
\newcommand{\frk}[1]{\{#1\}}

\newcommand{\transp}{\intercal}

\newcommand{\condNr}[1]{\frac{\la_{\max}(#1)}{\la_{\min}(#1)}}

\usepackage{lastpage}
\jmlrheading{..}{2025}{1-\pageref{LastPage}}{12/25; Revised
	../..}{../..}{..-....}{Stefan Perko}
\ShortHeadings{Towards Continuous-Time Approximations for SGDo}{Perko}

\firstpageno{1}

\begin{document}
	\title{Towards Continuous-Time Approximations for Stochastic Gradient Descent without Replacement}
	
	\author{\name Stefan Perko \email stefan.perko@uni-jena.de \\
		\addr Institute for Mathematics\\
		Friedrich-Schiller-University Jena\\
		07737 Jena, Germany}
	
	\editor{my editor}
	
	\maketitle

\begin{abstract}
Gradient optimization algorithms using epochs, that is those based on stochastic gradient descent without replacement (SGDo), are predominantly used to train machine learning models in practice. However, the mathematical theory of SGDo and related algorithms remain underexplored compared to their \enquote{with replacement} and \enquote{one-pass} counterparts.
In this article, we propose a stochastic, continuous-time approximation to SGDo with additive noise based on a \emph{Young differential equation} driven by a stochastic process we call an \emph{epoched Brownian motion}. We show its usefulness by proving the almost sure convergence of the continuous-time approximation for strongly convex objectives and learning rate schedules of the form $u_t = \frac{1}{(1+t)^\be}, \be \in (0,1)$. Moreover, we compute an upper bound on the asymptotic rate of almost sure convergence, which is as good or better than previous results for SGDo.
\end{abstract}


\section{Introduction}
Consider a risk minimization problem $(R : \R^d \times \cZ \to [0,\infty), \nu)$ on a measurable space $\cZ$. Fix an \tiid\ sequence $(\bz(n))_{n\in \N_0}$ in $\cZ$ with $\bz(0) \sim \nu$. For now, consider one-pass SGD with a sequence of learning rates $(\eta_n)_{n\in \N}$, given by
\begin{equation}
	\label{eq:introSGDlr}
	\chi_{n+1} = \chi_n - \eta_n \nabla R_{\bz({n})} (\chi_n), \quad h \in(0,1), n \in \N_0.
\end{equation}
In order to better understand SGD several authors have proposed approximating their dynamics by the solution of an SDE. In particular, in the case of a constant learning rate ($\eta_n = h$), \citet{mandt2015continuous} propose the following family of \tSDE s as an approximation of \eqref{eq:introSGDlr}
\[dY_t^h = - \nabla \cR(Y_t^h)\,dt + \sqrt{h}\si\,dW_t.\]

Here, $\si$ is a symmetric and positive semi-definite matrix approximating the gradient covariance in a \enquote{region of interest}, $W$ is a $d$-dimensional Brownian motion, and $\cR = \E R_{z(0)}$. Time is scaled in such a way that heuristically we have $Y_{nh}^h \approx \chi_n$.
Consider now a learning rate schedule $u : [0,\infty) \to [0,1]$ such that $\eta_n = h u_{nh}$.
\citet{Li15} further investigated this case of a non-constant learning rate schedules, and they heuristically used the following non-homogeneous dynamics
\begin{equation}
	\label{eq:SME1}
	dY_t^h = - u_t\nabla \cR(Y_t^h)\,dt + u_t\sqrt{h\Si(Y_t^h)}\,dW_t.
\end{equation}
The presence of $u$ in both coefficients can be motivated as follows. By multiplying the stochastic gradients with $u$, the expected gradients are multiplied by $u$ and their covariance by $u^2$. Thus, the diffusion coefficient - being the square root of the covariance is multiplied by $u$ as well.
While high learning rates seem to promise fast convergence via the drift, they also increase the variance of the gradients. A well-chosen learning rate schedule should thus balance both effects to ensure convergence.

Corollary 10 by \cite{Li18} implies that under certain regularity conditions \eqref{eq:SME1} is a first-order SME of SGD. However, by \citet[Theorem 6]{perko2024compare} we know that, among first-order SMEs, choosing a state-dependent diffusion coefficient is not always better than a state-independent one.
Therefore, in the following we elect to work with the simpler additive noise approximation of the form 
\begin{equation}
	\label{eq:introCCSGFu}
	dY_t^h = - u_t\nabla \cR(Y_t^h)\,dt + \sqrt{h}u_t\si\,dW_t,
\end{equation}
in the spirit of \citet{mandt2015continuous}.

The Markov property of Brownian motion says that the future is independent of the past given the current state. In the approximation \eqref{eq:SME1} this reflects the idea that all future data points of SGD are new data points, independent of those we have seen so far.

Consider now a finite \tiid\ sequence $(\bz(n))_{n=0}^{N-1}$ with $\bz(0)\sim \nu$, and the following variant of SGD, called SGD \emph{without replacement (with finite data)} (SGDo)
\begin{equation}
	\label{eq:genericSGDo}
	\chi_{n+1} = \chi_n - \eta_n \nabla R_{\bz({\pi^{\floor{n/N}}(n \modu N)})} (\chi_n), \quad n \in \N_0.
\end{equation}
Here, $(\pi^j)_{j\in \N_0}$ is a sequence of permutations of the set $\set{0,\dots, N-1}$. Wlog we set $\pi^0 = \id{}$.
Then the dynamics \eqref{eq:genericSGDo} and \eqref{eq:introSGDlr} coincide for $n\in \set{0,\dots, N-1}$. In the following \emph{epoch}, i.e.\ for $n \in \set{N,\dots, 2N-1}$, we reuse the same finite sample $(\bz(k))_{k=0}^{N-1}$ in perhaps a different order $(\bz(\pi^1(k)))_{k=0}^{N-1}$. We continue on like this in subsequent epochs using the sequence of permutations $(\pi^j)_{j\in \N_0}$. In general, we allow $(\pi^j)_{j\in \N_0}$ to be random, but independent of $(\bz(n))_{n=0}^{N-1}$.

For $t\in [0,T]$ with $T = Nh$, Equation \eqref{eq:introCCSGFu} is a reasonable approximation of \eqref{eq:genericSGDo}. However, Equation \eqref{eq:genericSGDo} no longer defines a Markov process for $n\geq N$ on the state space $\R^d$, because it cannot be written in the form $\chi_{n+1} = g(\chi_n, Z_n)$ for some \tiid\ sequence $(Z_n)_{n\in \N_0}$.
Thus, the Markov property for the driver $W$ in Equation \eqref{eq:introCCSGFu} is no longer appropriate if we try to find a continuous-time model for SGDo (for finite data).

For now, let us consider \emph{single-shuffle} SGDo, that is we choose\footnote{Technically, in the literature on SGDo \enquote{single shuffle} means \enquote{shuffle once}. We assume no shuffling here because it makes no difference: the distribution of the sample is unaffected.} $\pi^j = \id{}, j \geq 1$.
Given $T > 0$ and a Brownian motion $W : \Om \times [0,T]\to \R^d$, define
\[\hat W_t := W_{\frk{t/T}T} + \floor{t/T}W_T, \quad t\geq 0.\]
Here, $\frk r = r - \floor r$ is the fractional part of $r\in \R$.
Note that $\hat W$ is a Brownian motion when restricted to the interval $[0,T)$, and $\hat W$ satisfies
\begin{equation*}
	\label{eq:epochProperty}
	\hat W_{t+jT} = \hat W_t + jW_T, \qquad t \geq 0, j\in \N_0.
\end{equation*}
Note that $\hat W$ is almost surely continuous and even locally Hölder continuous.
The increments of $\hat W$ on $[jT, (j+1)T]$ coincide with the increments of $W$ on $[0, T]$ (up to translating time).  
We call $\hat W$ a \emph{single shuffle Brownian motion} with period $T$. The fact that we reuse the same Brownian path $(W_t)_{t\in [0,T]}$ corresponds to using the same data points in the same order in later epochs (single-shuffle).

By replacing the driving path of the diffusion in \eqref{eq:introCCSGFu} by single shuffle Brownian motion, we arrive at the following differential equation with additive noise
\begin{equation}
	\label{eq:epochedOU}
	dY_t = - u_t \nabla \cR(Y_t) \,dt + u_t \sqrt h \si\,d\hat W_t.
\end{equation}
Since $\hat W$ is not a semimartingale we cannot interpret the term $ u_t\,d\hat W_t$ using \tIto{} integration. Instead, we interpret it pathwise as the Young integral 
\[\int_0^t u_s \,d\hat W_s = \lim_{|\cP| \to 0} \sum_{[r,s]\in \cP} u_r (\hat W_s - \hat W_r),\]
where the limit is taken with respect to all partitions of $[0,t]$ with mesh size $|\cP|$. The integral exists for example if $u$ is Lipschitz. Thus, we understand \eqref{eq:epochedOU} as Young differential equation.

More generally, we allow the driver $\hat W$ in Equation \eqref{eq:epochedOU} to be an \emph{epoched Brownian motion} (EBM). An EBM $\hat W$ is roughly speaking a single shuffle Brownian motion, except on $[jT, (j+1)T]$ the increments of $\hat W$ may be \enquote{infinitesimally shuffled} according to $\pi^j$ (see Section \ref{sec:ebms} for a proper explanation). We can thereby encode different shuffling schemes for SGDo in the approximating equation \eqref{eq:epochedOU}.

\subsection{Summary of Contributions}
Below we provide a summary of the main contributions of this paper.
\begin{itemize}
	\item We introduce the Young differential equation \eqref{eq:epochedOU} as a stochastic, continuous-time approximation to SGD without replacement in the finite-data setting, for large sample sizes.
	\item We motivate the general class of epoched Brownian motions (EBM) as drivers of Equation \eqref{eq:epochedOU} and discuss their correspondence to different shuffling schemes for SGDo.
	\item To demonstrate the usefulness of our heuristic approximation, we study the almost sure convergence of the solution of \eqref{eq:epochedOU} for Lipschitz and strongly convex $\cR$ with Hölder continuous Hessian matrix, and with $u_t = \frac{1}{(1+ct)^\be}, t\geq 0, \be \in (0,1), c > 0$. Here, we leave out the case $\be = 1$ for brevity reasons. In contrast to previous works however, we cover the case $\be \in (0,1/2]$ as well. This is because our main strategy uses the Young-Lóeve inequality instead of martingale techniques.
	\item We show convergence to a random point depending on $\hat W_T$ and compute an asymptotic upper bound on the convergence speed. Our result for the single shuffle case matches previous results by \citet{gurbuzbalaban_why_2021}. In the case of general random permutations, our results suggest markedly better upper bounds than the best results known for random reshuffling. Note that, heuristically speaking, $\hat W_T$ encodes information about the random sample $(\bz(n))_{n=1}^N$ including the sample size $N$, which is why the limit depends on it. In the setting of linear regression, we identify the random limit with the (random) OLS estimator, which further substantiates the legitimacy of our approximation.
\end{itemize}

\subsection{Related Work}
The idea to use \tSDE s for approximating SGD processes was first considered by \citet{mandt2015continuous} and \citet{Li15, Li18}. \citet{mandt2015continuous} heuristically use an SDE with additive noise for approximating and analyzing the SGD process. \citet{Li15} derived a SDE with multiplicative noise and rigorously proved that it is a first-order approximation of SGD \citep{Li18} with respect to the learning rate $h$. \citet{perko2024compare} show that gradient flow and they approximations by \citet{mandt2015continuous} and \citet{Li18} are first-order approximations of SGD, even for time-dependent learning rates. \citet[Chapter 7]{perko_modified_2025} (in particular Theorem 7.6.1.) shows that epoched Brownian motions arise as weak scaling limits of random walks with finitely many distinct increments.

Many previous works on SGDo \citep{shamir2016,nagaraj2019,nguyen_unied_nodate,rajput_closing_2020,rajput_permutation-based_2021,mishchenko_random_2020,ahn_sgd_2020,jain_sgd_2020,koren_benign_nodate,gurbuzbalaban_why_2021} have established various upper and lower bounds on the convergence rates \emph{in expectation} in various settings.
Moreover, \citet{ahn_sgd_2020} also establish high probability upper bounds on convergence rate of SGDo. \citet{li_unified_2022} prove almost sure convergence of the SGDo gradients for square-summable learning rates.

\citet{gurbuzbalaban_why_2021} also proves almost sure convergence for single-shuffle and random reshuffling SGDo. The later algorithm uses an \tiid\ sequence $(\pi^j)_{j\in \N_0}$ of permutations where $\pi^0$ uniformly distributed. Using martingale techniques, they an asymptotic upper bound on the almost sure convergence rates for learning rates decaying like the schedule $u_t = \frac{1}{(1+t)^\be}, t\geq 0$ with $\be \in(1/2,1]$, and strongly convex objective function $\cR$.

This article significantly expands on the ideas in the unpublished preprint by \citet{perko2022towards}.

\section{SMEs driven by epoched Brownian motions}
\label{sec:ebms}
Let $(\Om, \cF_\Om, \P)$ be a complete probability space, $d\in \N$ and $T > 0$.
Recall that $\hat W$ is a single shuffle Brownian motion (of period $T$) if there exists a Brownian motion $W : \Om \times [0,T]\to \R^d$ with
\[\hat W_t := W_{\frk{t/T}T} + \floor{t/T}W_T, \quad t\geq 0.\]

Note that given a single shuffle Brownian motion $\hat W$ we can define a \emph{Brownian bridge} $B : \Om \times [0,1] \to \R^d$ from $0$ to $0$ by setting
\[B_t = \frac{1}{\sqrt T}(\hat W_{tT} - t\hat W_T),\quad t\in [0,1].\]
Then,
\[\hat W_t = \sqrt T B_{\frk{t/T}} + \frac t {\sqrt T} V, \quad t\geq 0.\]
with $V := \frac{1}{\sqrt T} \hat W_T$ a standard Gaussian.

More generally, we may replace the single Brownian bridge $B$ with a sequence of bridges $(B^j)_{j\in \N}$, one for each epoch. This motivates the following definition.
\begin{definition}
	A stochastic process $X : \Om \times [0,\infty) \to \R^d$ is called an \emph{epoched Brownian bridge} if there exists a jointly Gaussian\footnote{Jointly Gaussian family means $(B^{j_1}_{t_1}, \dots, B^{j_m}_{t_m})$ is Gaussian for all $j_1,\dots, j_m \in \N_0$ and $t_1,\dots, t_m \in [0,1]$.} family $(B^j : \Om \times [0,1] \to \R^d)_{j \in \N_0}$ of Brownian bridges from $0$ to $0$, such that
	\[X_t = B^{\floor t}_{\frk t}, \quad t\geq 0.\]
	A stochastic process $\hat W : \Om \times [0,\infty) \to \R^d$ is called an \emph{epoched Brownian motion} of period $T > 0$ if there exists an epoched Brownian bridge $X$ and a random variable $V\sim \cN(0, 1_{d\times d})$ independent of $X$, such that
	\[\hat W_t = \sqrt T X_{t/T} + \frac{t}{\sqrt T} V, \quad t\geq 0.\]
\end{definition}

We highlight the following examples:
\begin{enumerate}[(a)]
	\item Single shuffle (SS): $B^0 = B^1 = \dots$,
	\item Random reshuffling (RR): $(B^j)_{j \in \N_0}$ are independent,
	\item Flip-flop single shuffle: $B^0 = B^2 = \dots$, and $B^{j+1}_t = -B^j_{1-t}, t\in [0,1]$,
	\item Flip-flop random reshuffling: $(B^{2j})_{j\in \N_0}$ are independent, $B^{j+1}_t = -B^j_{1-t}, t\in [0,1]$.
\end{enumerate}

In our framework, the epoched Brownian motion $\hat W$ corresponds to the versions of SGDo with the same name. That is, they correspond to the following shuffling schemes for SGDo for large samples sizes $N$:
\begin{enumerate}[(a)]
	\item Single shuffle (SS): $\pi^j = \id{N}, j \in \N$,
	\item Random reshuffling (RR): $(\pi^j)_{j \in \N_0}$ are independent with $\pi^j$ uniformly distributed on the symmetric group of order $N$,
	\item Flip-flop single shuffle: $\pi^{2j} = \id{N}, \pi^{2j+1} = \tau, j \in \N_0$, where $\tau(n) = N - n + 1$ is the \emph{reversal} permutation\footnote{Not to be confused with the inverse of a permutation.},
	\item Flip-flop random reshuffling: $(\pi^{2j})_{j \in \N_0}$ are independent with $\pi^j$ uniformly distributed on the symmetric group of order $N$, and $\pi^{2j+1} = \tau \circ \pi^{2j}, j \in \N_0$.
\end{enumerate}

We do not claim that every epoched Brownian motion or bridge correspond to a shuffling scheme for SGDo. Instead, a \emph{one-dimensional} epoched Brownian motion (or bridge) given by a family of Brownian bridges $(B^n : \Om \times [0,1]\to \R)_{n\in \N_0}$ corresponds to a shuffling scheme for SGDo for large sample sizes $N$ if there exists a measure $\mu$ on $[0,1]^\N$ with uniform marginals, such that
\[\E[B_s^i B_t^j] = C^{ij}(s,t) - st, \quad i \neq j \in \N, s,t\in [0,1],\]
where
\[C^{ij}(s,t) = \mu([0,1]\times \dots \times [0,1] \times \overbrace{[0,s]}^i \times [0,1]\times \dots \times [0,1] \times \overbrace{[0,t]}^j \times [0,1] \times \dots), \quad i\neq j\]
and $C^{ii}(s,t) = s\wedge t$, $i\in \N$.
Note that the functions $C^{ij}$ are $2$-copulas. A $d$-dimensional epoched Brownian bridge corresponding to a shuffling scheme consists of $d$ independent copies of such a one-dimensional process (the same measure is used for all dimensions).

The reason we claim correspondence to shuffling schemes, provided such a measure $\mu$ exists, is that these processes arise as scaling limits of the joint distributions of random walks that have the same increments, up to a (random) permutation, see \citet[Chapter 7, Theorem 7.6.1.] {perko_modified_2025}.

All our previous examples satisfy this condition, with
\begin{enumerate}[(a)]
	\item Single Shuffle (SS): $C^{ij}(s,t) = s\wedge t$,
	\item Random reshuffling (RR): $C^{ij}(s,t) = st$,
	\item Flip-flop single shuffle: 
	\[C^{ij}(s,t) = \begin{cases}
		s\wedge t, & i,j \text{ are both odd or even},\\
		(s+t-1)\vee 0, & \text{else}, 
	\end{cases}\]
	\item Flip-flop random reshuffling: 
	\[C^{ij}(s,t) = \begin{cases}
		(s+t-1)\vee 0, & i\text{ is even and } i+1 = j, \\
		st, & \text{else},
	\end{cases}\]
\end{enumerate}
for $i\neq j$.

The first formula is simply stating that the covariance of a single Brownian bridge is given by 
\[\Cov(B_s,B_t) = s\wedge t - st = s(1-t) \wedge t(1-s), \quad s,t\in [0,1].\]
The second formula just says that independent Brownian bridges have covariance $0$. To show (c) and (d) it remains the consider a Brownian bridge $B$ and calculate
\begin{align*}
	\Cov(B_s,-B_{1-t}) = & -(s\wedge(1-t)) + s(1-t) \\
	= & (-s)\vee (t-1) + s - st \\
	= &(s+t-1)\vee 0 - st, \quad s,t\in [0,1].
\end{align*}

Since most of our results do not depend on the existence of such a measure $\mu$ we will not assume such a covariance structure in general. 

\section{Main result}
Let $d\in \N$ and $\la > 0$. We say a function $\cR : \R^d \to \R \in \dC 2$ is \emph{$\la$-strongly convex} if it satisfies any of the following equivalent properties:
\begin{itemize}
	\item $\innp{\nabla \cR(x) - \nabla \cR(y)}{x-y} \geq \la |x-y|^2, \quad x,y\in \R^d$,
	\item $\cR(y) \geq \cR(x) + \innp{\nabla \cR(x)}{y-x} + \frac12\la |x-y|^2, \quad x,y\in \R^d$,
	\item $\nabla^2\cR(x) - \la 1_{d\times d}$ is a positive semi-definite matrix, for all $x\in \R^d$.
\end{itemize}
Here, $\nabla^2 \cR$ denotes the Hessian of $\cR$.
Let $L > 0$. We say $\cR$ is \emph{$L$-smooth} if $\nabla \cR$ is Lipschitz, with $\nrm{\nabla \cR}{\Lip} \leq L$.
Our main (mathematical) result is the following.

\begin{theorem}
	\label{thm:main}
	Let $\be \in (0,1)$, $c > 0$, $L, \la > 0$ and $\cR : \R^d\to \R \in \dC 2$ be $\la$-strongly convex and $L$-smooth such that $\nabla^2 \cR$ is Hölder continuous.
	Let $Y$ be the solution to the Young differential equation
	\begin{equation}
		\label{eq:epochedCvx}
		dY_t = - \frac{1}{(1+ct)^\be} \nabla \cR(Y_t)\,dt + \frac{1}{(1+ct)^\be} \si\,d\hat W_t,
	\end{equation}
	driven by an epoched Brownian motion $\hat W$ with period $T$. Then
	\[\left|Y_t - (\nabla \cR)^{-1}(T^{-1}\si \hat W_T)\right| \leq T^{1/2-\be} |\si| \left(4.7\frac{L}{\la} + 1.2\right)c^{-\be}\frac{\sqrt{\log t}}{t^{\be}} + o\left(\sqrt{\log t}\cdot  t^{-\be}\right), t\to \infty, \quad a.s.\]
\end{theorem}
Theorem \ref{thm:main} may give the impression that its optimal to let $\be\to 1$-. After all, that choice gives us the fastest asymptotic rate of convergence. However, in actuality the constant hidden in $o(\sqrt{\log t}\cdot  t^{-\be})$ diverges to $\infty$, as $\be \to 1$. Therefore, we cannot conclude that $\be \to 1$ is optimal. In fact, in practice setting $\be = 1$ makes the learning rates decay much too fast.

In certain situations we can get a better decay rate compared to Theorem \ref{thm:main}.
The following theorem applies to all epoched Brownian motions which have only finitely many different epochs over their entire time horizon. For example, this is the case for single shuffle Brownian motion, which only has a single repeated epoch.
\begin{theorem}
	\label{thm:mainAlt}
	Let $\be \in (0,1)$, $c > 0$, $L, \la > 0$ and $\cR : \R^d\to \R \in \dC 2$ be $\la$-strongly convex and $L$-smooth, such that $\nabla^2 \cR$ is Hölder continuous.
	Let $Y$ be the solution to the Young differential equation
	\begin{equation}
		dY_t = - \frac{1}{(1+ct)^\be} \nabla \cR(Y_t)\,dt + \frac{1}{(1+ct)^\be} \si\,d\hat W_t,
	\end{equation}
	driven by an epoched Brownian motion $\hat W$ with period $T$. Suppose further there exists a number $J\in \N$, such that $\cI := \set{(\hat W_{(j+t)T}-\hat W_{jT})_{t\in [0,1]} : j\in \N}|$ satisfies $|\cI| = J$, almost surely. Then, for all $\al \in(0,1/2)$,
	\[\left|Y_t - (\nabla \cR)^{-1}(T^{-1}\si \hat W_T)\right| \leq C_\al T^{1/2-\be} |\si| \left(\frac{1}{1-2^{-\al}} \frac{L}{\la} + 1\right)\frac{1}{t^{\be}} + o\left(C_\al t^{-\be}\right),  t\to \infty, \quad a.s.\]
	where $C_\al = \max_{w\in \cI} \nrm{w}{\al}$.
\end{theorem}
Note that the only random factor in $o(C_\al t^{-\be})$ is $C_\al$.

As an example, consider SGDo applied to linear regression, which corresponds to the \tYDE{}
\[dY_t = - \frac{1}{(1+t)^\be} \ka (Y_t - \thet^*)\,dt + \frac{1}{(1+t)^\be} \sqrt{h \si_\ep^2 \ka}\,d\hat W_t.\]
Here, $\hat W$ has period $T = Nh$ where $N$ is the sample size and $h$ the maximal learning rate. We implicitly assume we are in the underparameterized regime $N\gg d$.

Then
\begin{align*}
	(\nabla \cR)^{-1}(T^{-1}\si \hat W_T) = & \thet^* + \ka^{-1}((Nh)^{-1/2}\sqrt{h \si_\ep^2 \ka} T^{-1/2}\hat W_{T}) \\
	= & \thet^* + \frac{\si_\ep}{\sqrt N} \ka^{-1/2} (T^{-1/2}\hat W_T)\\
	\sim &\cN\left(\thet^*, \frac{\si_\ep^2}{N} \ka^{-1}\right),
\end{align*}
and Theorem \ref{thm:main} implies
\begin{align*}
	\left|Y_t - \left(\thet^* + \frac{\si_\ep}{\sqrt N} \ka^{-1/2} (T^{-1/2}\hat W_T)\right)\right| \leq & (Nh)^{1/2-\be} \sqrt h \si_\ep|\sqrt \ka| \left(4.7\condNr{\ka}+1.2\right)c^{-\be}\frac{\sqrt{\log t}}{t^{\be}}\\
	& + o\left(\sqrt{\log t}\cdot  t^{-\be}\right)\\
	\leq & N^{1/2-\be}d h^{1-\be} \si_\ep  \sqrt{\la_{\max}(\ka)} \left(4.7\condNr{\ka}+1.2\right)\frac{\sqrt{\log t}}{t^{\be}} \\
	&+ o\left(\sqrt{\log t}\cdot  t^{-\be}\right),
\end{align*}
as $t\to \infty$, almost surely. 
The limit $Y_\infty := \thet^* + \frac{\si_\ep}{\sqrt N} \ka^{-1/2} T^{-1/2}\hat W_T$ of $Y$ has the same mean and covariance matrix as the OLS estimator
\[\hat\thet = \left(\sum_{n=1}^N \bx_n \bx_n^\transp\right)^{-1}\left(\sum_{n=1}^N \bx_n \by_n\right),\]
if $(\bx_n, \by_n)_{n=1}^N$ is a finite \tiid\ sample with $(\bx_0,\by_0) \sim \nu$, and $\nu$ is the corresponding population.
Since $\hat W$ is independent of $(\bx_n,\by_n)_{n\in \N}$ we \emph{do not} have $\hat \thet = Y_\infty$, even if $\hat \thet$ was Gaussian. Nevertheless, this result suggests that spiritually $Y_\infty$ represents the OLS estimator in our model in the case of linear regression.

The factor $T^{1/2-\be}$ (or $N^{1/2-\be}$ after setting $T = Nh$) in the convergence speed may be surprising. It can be heuristically explained as follows:
Set $u_t = \frac{1}{(1 + ct)^\be}, t\geq 0$. The noise accumulated in epoch $j$ is given by
\[\int_{jT}^{(j+1)T} u_t \si \,d\hat W_t \approx (cjT)^{-\be} \si (\hat W_{(j+1)T} - \hat W_{jT}) = T^{1/2-\be}(jc)^{-\be} \si Z,\]
where
\[Z = \frac{1}{\sqrt{T}}(\hat W_{(j+1)T} - \hat W_{jT}) \sim \cN(0, 1_{d\times d}).\]
If $\be > 1/2$, then $u$ decays faster than the noise accumulates. In this case the accumulated noise vanishes, as $T\to \infty$, since increasing $T$ means we are effectively averaging over more \tiid\ random variables per epoch. On the other hand, if $\be < 1/2$, then $u$ decays too slowly to overcome the noise accumulation. More steps per epoch means more accumulation, so the accumulated noise diverges to infinity, as $T\to \infty$. Finally, at $\be = 1/2$ both effects (decay and noise accumulation) are balanced. 

These different regimes implicitly also exist in other works on stochastic gradient descent (with or without replacement). In particular, usually only the case $\be > 1/2$ is covered (see the end of the following paragraph).

\paragraph{Comparison with existing results}
Our main theorem complements findings by\\\citet{gurbuzbalaban_why_2021}.
They proved that single shuffle SGDo satisfies 
\[|\chi_n - \hat\thet| \leq \frac{h |\mu(\pi^1)|}{\la} \frac{1}{n^{\be}} + o(n^{-\be}), a.s. \quad k\to \infty,\]
for $\be \in (1/2,1)$.
Here, $\chi$ is given by Equation \ref{eq:genericSGDo} with $\eta_n = hn^{-\be}$ and $\pi^1 = \pi^j, j \in \N$. Further, $\cR$ is given as a sum of $N$ quadratic forms, is $\la$-strongly convex and has its minimum at $\hat\thet$. Moreover, $\mu(\pi) \in \R^d$ is a sum of $\frac12 N(N-1)$ terms depending on $\cR$ and the permutation $\pi$. In general, $|\mu(\pi)|$ can grow with rate $O(N^2)$, as $N\to \infty$.
In contrast, Theorem \ref{thm:mainAlt} suggests a rate of
\[\tilde C N^{1/2-\be} n^{-\be} + o(n^{-\be}), a.s. \quad k\to \infty.\]
where $\tilde C$ is independent of $N$.
They also provide a crude bound for the random reshuffling case:
\[|\chi_k - \hat\thet| \leq \frac{h \sup_{\pi \in \Sym N}|\mu(\pi)|}{\la} \frac{1}{n^{\be}} + o(n^{-\be}), a.s. \quad k\to \infty,\]
where $\Sym N$ is the symmetric group of degree $N$.
However, in the worst case $\sup_{\pi \in \Sym N}|\mu(\pi)| = O(N^2 N!)$, as $N\to \infty$, making this result not very useful for moderately large $N$, say\footnote{The observable universe is estimated to have less than $60!$ particles.} $N > 100$. Naturally, they mention that the constant $\sup_{\pi \in \Sym N}|\mu(\pi)|$ is pessimistic.
Our Theorem \ref{thm:main} suggests a rate of
\[\tilde C N^{1/2-\be}\frac{\sqrt{\log n}}{n^{\be}} + o(\sqrt{\log n}\cdot n^{-\be}), a.s. \quad k\to \infty,\]
for the convergence of SGDo on strongly convex objectives using any shuffling scheme. 
Thus, Theorem \ref{thm:main} suggests good almost sure convergence rates for SGDo even for large sample sizes $N$.

Finally, note the restriction $\be > 1/2$ imposed by \citet{gurbuzbalaban_why_2021}. It stems from the application of martingale techniques which require learning rates to be square summable. Indeed, 
\[\sum_{n=1}^\infty \left(\frac{1}{n^\be}\right)^2 < \infty\text{ if and only if } \be > 1/2.\]
Since we do not use any martingale techniques, this barrier only appears implicitly in our main results as the convergence rate factor $T^{1/2-\be}$.
\newcommand{\rpX}{\mathbf X}
\newcommand{\rpaX}{\mathbb X}

\section{Properties of (epoched) Brownian bridges}
In the following we will mostly work with epoched Brownian \emph{bridges}. By the definition they concatenations of Brownian bridges. Recall, that a Brownian bridge is $(1/2-)$-Hölder continuous, that is $(1/2-\ep)$-Hölder continuous for every $\ep > 0$. Together with the following lemma, this implies that epoched Brownian bridges are locally $(1/2-)$-Hölder continuous.

Let $\al \in (0,1)$. We denote by $\nrm{\blnk}{\al}$ the $\al$-Hölder seminorm given by
\[\nrm{f}{\al} = \sup_{s,t\in I} \frac{\nrm{f(t)-f(s)}{E}}{|t-s|^\al},\]
where $f : I \to E$ for $E = (\R^d, |\blnk|)$ or $E = (\R^{d\times d}, \specnrm{\blnk})$ and some interval $I$. Here,
\[\specnrm{A} := \sup_{|x| = 1} |Ax| = \sqrt{\la_{\max}(A^\transp A)}.\]
denotes the \emph{spectral norm} of a square matrix $A$. Further, $\blnk$ denotes a placeholder for an argument. We also write $\nrm{f}{\al;I} = \nrm{f|_I}{\al}$ when $f$ is defined on a set containing $I$. In the case $\al = 1$ we prefer writing $\nrm{f}{\Lip}$ and $\nrm{f}{\Lip;I}$. We introduce the following function spaces:
\begin{itemize}
\item $\hoelHZ \al$ - $\al$-Hölder continuous functions,
\item $\Lip$ - Lipschitz continuous functions,
\item $\CHoelLoc\al$ - locally $\al$-Hölder continuous functions,
\item $\CHoelLoc{\al^-}$ - locally $(\al-)$-Hölder continuous functions,
\item $L^1_{\operatorname{loc}}$ - locally integrable functions.
\end{itemize}
\begin{lemma}
	\label{lem:concatHoel}
	Let $\al \in (0,1)$ and $f,g : [0,1]\to \R^d \in \hoelHZ \al$ be functions with $f(1) = g(0)$. Then the concatenation
	\[f\ast g : [0,2]\to \R^d, t\mapsto f(t)1_{[0,1]}(t) + g(t-1)1_{(1,2]}(t) \] 
	satisfies $f\ast g\in \hoelHZ \al$ with $\nrm{f\ast g}{\al} \leq 2^{1-\al} (\nrm{f}{\al} \vee \nrm{g}{\al})$.
\end{lemma}
\begin{proof}
	It suffices to check the Hölder condition for $s < 1 < t$. In this case
	\begin{align*}
		|f\ast g(t) - f\ast g(s)| \leq & |f\ast g(t) - f\ast g(1)| + |f\ast g(1) - f\ast g(s)| \\
		= & |g(t-1) - g(0)| + |f(1) - f(s)|\\
		\leq &(\nrm{f}{\al} \vee \nrm{g}{\al})(|t-1|^\al + |1-s|^\al)\\
		\leq &2^{1-\al} (\nrm{f}{\al} \vee \nrm{g}{\al})(|t-1| + |1-s|)^\al\\
		= &2^{1-\al}(\nrm{f}{\al} \vee \nrm{g}{\al})|t-s|^\al,
	\end{align*} 
	since $|t-1| + |1-s| = t-1 + 1 - s$.
\end{proof}

\newcommand{\sphre}[1]{\mathbb S^{#1}}
\begin{lemma}[Borell-TIS]
	\label{lem:borellTIS}
	Let $D$ be a topological space and $Q : \Om \times D \to \R^d$ be Gaussian random field, which is almost surely bounded on $D$. 
	Define $m = \E\left[\sup_{t\in D} |Q_t|\right]$ and $\si^2 = \sup_{t\in D} \la_{\max}(\Cov(Q_t))$.
	Then
	\[\P\left(\sup_{t\in D} |Q_t| > x\right) \leq e^{-\frac{(x-m)^2}{2\si^2}}, \quad x > m.\]
\end{lemma}
\begin{proof}
	We write $\sphre{d-1} = \set{v\in \R^d : |v| = 1}$. Note that
	\[|Q_t| = \sup_{v\in \sphre{d-1}} \innp{Q_t}{v},\]
	since $|\innp{Q_t}{v}| \leq |Q_t||v| = |Q_t|$ for $v\in \sphre{d-1}$ and because we can pick $v = Q_t/|Q_t|$.
	Define
	\[\tilde Q : \Om \times D \times \sphre{d-1} \to \R, (\om, t, v) \mapsto \innp{Q_t(\om)}{v}.\]
	Then $\tilde Q$ is again a Gaussian random field and almost surely bounded. We have 
	\[\E\left[\sup_{(t,v)\in D \times \sphre{d-1}} \tilde Q_{t,v}\right] = m.\]
	Moreover, we have $\Var(\innp{Q_t}{v}) = v^\transp \Cov(Q_t) v$, and so
	\[\sup_{(t,v)\in D \times \sphre{d-1}} \Var(\innp{Q_t}{v}) = \sup_{t\in D} \sup_{v\in \sphre{d-1}} v^\transp \Cov(Q_t) v = \sup_{t\in D} \la_{\max}(\Cov( Q_t)) = \si^2.\]
	The penultimate equality follows because we are maximizing the Rayleigh quotient of $\Cov(Q_t)$.
	Now, using the standard Borell-TIS inequality \citep[see][Theorem 2.1.1]{adler2009random} we have 
	\[\P\left(\sup_{(t,v)\in D \times \sphre{d-1}} \tilde Q_{t,v} - m > x\right) \leq e^{-\frac{x^2}{2 \si^2}}, \quad x > 0,\]
	or equivalently
	\[\P\left(\sup_{t\in D} |Q_t| > x\right) \leq e^{-\frac{(x-m)^2}{2\si^2}}, \quad x > m.\]
\end{proof}

\begin{lemma}
	\label{lem:expttnTailFrmla}
	Let $g : [0,\infty) \to \R \in \dC 1$ and $Z$ be a non-negative random variable. Then
	\[\E g(Z) = g(0) + \int_0^\infty g'(x)\P(Z > x)\,dx.\]
\end{lemma}
\begin{proof}
	We have
	\[g(z) = g(0) + \int_0^z g'(x)\,dx,\]
	and so  
	\[\E g(Z) = g(0) + \E\left[\int_0^Z g'(x)dx\right] = g(0) + \int_0^\infty g'(x)\P(Z > x)\,dx.\]
\end{proof}

\begin{lemma}
	\label{lem:BBferniq}
	Let $B : \Om \times [0,1] \to \R^d$ be a Brownian Bridge. Then 
	\[\E[e^{a \nrm{B}{\al}^2}] < \infty\]
	for all $\al \in (0,1/2)$ and $a \in (0, \frac{1}{2(1-b)b^{1-2\al}})$, where $b = \frac{1-2\al}{2-2\al}$.
\end{lemma}

\begin{proof}
	Define 
	\[Q_{s,t} = \begin{cases}
		\frac{B_t - B_s}{|t-s|^\al}, & s\neq t,\\
		0, & s = t,
	\end{cases}\]
	for all $s,t\in [0,1]$, and write $\hat Q := \sup_{s,t\in [0,1]} Q_{s,t}$.
	Then $Q$ is a Gaussian random field $\Om \times [0,1]^2\to \R^d$ and $\sup_{s,t\in [0,1]} |Q_{s,t}| = \nrm{B}{\al}$. Thus, by Lemma \ref{lem:borellTIS}
	\begin{align*}
		\P(\nrm{B}{\al} > x) \leq e^{-\frac{(x-m)^2}{2 \si^2}}, \quad x > m := \E\nrm{B}{\al},
	\end{align*}
	where $\si^2 := \sup_{s,t\in [0,1]} \la_{\max}(\Cov Q_{s,t})$. Because the components of $B$ are independent, Brownian bridges have stationary increments and using the covariance formula for a one-dimensional Brownian bridge we have
	\[\la_{\max}(\Cov(B_t - B_s)) = \Var(B_t^1 - B_s^1) = \Var(B_{t-s}^1) = |t-s|(1 - |t-s|), \quad s,t\in [0,1].\]
	Thus,
	\[\la_{\max}(\Cov Q_{s,t}) = \begin{cases}
		\frac{|t-s|(1 - |t-s|)}{|t-s|^{2\al}},  &s\neq t,\\
		0, & s = t
	\end{cases} = f(|t-s|), \quad s,t\in [0,1],\]
	where $f(b) = (1 - b)b^{1-2\al}$. The function $f$ attains its maximum at $b^* := \frac{1-2\al}{2-2\al}$. Hence $\si^2 = f(b^*)$.
	Let $a > 0$. Then Lemma \ref{lem:expttnTailFrmla} implies
	\[\E[e^{a \nrm{B}{\al}^2}] = 1 + \int_0^\infty 2 a x e^{ax^2} \P(\nrm{B}{\al} > x)\,dx.\]
	Estimating the tail of the integral, we have
	\[\int_m^\infty 2 a x e^{ax^2} \P(\nrm{B}{\al} > x)\,dx \leq \int_m^\infty 2 a x e^{ax^2}e^{-\frac{(x-m)^2}{2 \si^2}}\,dx.\]
	Since
	\[ax^2 -\frac{(x-m)^2}{2 \si^2} = \left(a - \frac{1}{2 \si^2}\right) x^2 + \frac{m}{\si^2}x - \frac{m^2}{2\si^2}\]
	the integral converges if $a < \frac{1}{2 \si^2} = \frac{1}{2f(b^*)}$.
\end{proof}

The following lemma gives us one factor in the decay rate of Theorem \ref{thm:main}.
\begin{lemma}
	\label{lem:epochedBBgrowth}
	Let $\al \in (0,1/2)$, $a \in (0, \frac{1}{2(1-b)b^{1-2\al}})$, where $b = \frac{1-2\al}{2-2\al}$, and $(B^j)_{n\in \N_0}$ be a family of Brownian bridges.
	Then
	\[\max_{j \leq n}\nrm{B^j}{\alpha} \leq a^{-1/2}\sqrt{\log n},\]
	for large $n\in \N$, almost surely.
\end{lemma}
\begin{proof}
	We use Lemma \ref{lem:BBferniq}.
	By Markov's inequality
	\[\P(\nrm{B}{\alpha} \geq x) = \P(e^{a \nrm{B}{\alpha}^2} \geq e^{ax^2}) \leq \E[e^{a \nrm{B}{\alpha}^2}] e^{-a x^2},\]
	for all $x\in \R$.
	Define $Z_j = \nrm{B^j}{\alpha}, j \in \N$, and $Z_n^* = \max(Z_1,\dots, Z_n)$.
	Then
	\[\P(Z_n^* > x) \leq \sum_{j=1}^n \P(Z_j > x) \lesssim n e^{-a x^2},\]
	uniformly over $x$ and $n$. For any $\ep > 0$ we thus have
	\[\sum_{j=1}^\infty \P(Z_{2^j}^*> \sqrt{\frac{1+\ep}{c}\log 2^j}) \lesssim \sum_{j=1}^\infty 2^{-j\ep} < \infty.\]
	By Borel-Cantelli
	\[\P\left(\limsup_{n\to \infty} \set{Z_n^* > \sqrt{\frac{1+\ep}{a}\log n}}\right) = 0,\]
	that is
	\[\max_{j \leq n}\nrm{B^j}{\alpha} = Z_n^* \leq \sqrt{\frac{1+\ep}{a}\log n},\]
	for large $n\in \N$, almost surely. Finally, by picking a slightly smaller $a$ we can leave out the $+\ep$. However, since we started with an arbitrary $a < \frac{1}{2(1-b)b^{1-2\al}}$ we have
	\[\max_{j \leq n}\nrm{B^j}{\alpha}\leq a^{-1/2}\sqrt{\log n},\]
	for large $n\in \N$, almost surely, for all $a \in (0, \frac{1}{2(1-b)b^{1-2\al}})$.
\end{proof}

\begin{docu}
	This is a conjecture, but mostly because I dont know the proper theory of 2D-Young (?) integration.
	We say $(B,B)' :  \Om \times [0,1]\to \R$ is a $C$-Brownian bridge if it is a Gaussian process, $B,B'$ are Brownian bridges from $0$ to $0$ and their covariance satisfies
	\[\Cov(B_s, B_t') = C(s,t) - st.\]
	\begin{proposition}
		Let $C$ be a $2$-copula, $(B,B')$ be a $C$-Brownian bridge, $\be \in (0,1)$ with $\frac12 + \frac1\beta > 1$ and $f,g\in C^{\be}([0,1])$. Then
		\[\Cov\left(\int_0^1 f(t)\,dB_t, \int_0^1 g(t)\,dB_t'\right) = \Cov_{(U,V)\sim C}(f(U),g(V)).\]
	\end{proposition}
	\begin{proof}[Proof sketch]
		We have
		\begin{align*}
			\E[\int_0^1 f(t)\,dB_t, \int_0^1 g(t)\,dB_t'] = & \lim_{|\cP_1|, |\cP_2| \downarrow 0} \sum_{(r,s) \in \cP_1, (t,u) \in \cP_2} f(r)g(t) \E[B_{r,s} B'_{t,u}] \\
			= & \int_{[0,1]^2} f(s)g(t) d\E[B_s B'_t]\\
			= & \int_{[0,1]^2} f(s)g(t) d(C(s,t) - st)\\
			= & \E[f(U)g(V)] - \E[f(U)] \E[g(V)],
		\end{align*}
		where $(U,V)\sim C$.
	\end{proof}
\end{docu}

\section{Young differential equations driven by epoched noise}
In this section we study the properties of Young differential equations with state-independent noise term, specifically driven by an epoched bridge $X$. Let $m \in \N$. We call $X : [0,\infty) \to \R^m$ an \emph{epoched bridge} if $X$ is locally Hölder continuous and $X_n = 0, n\in \N$. None of the arguments in this section directly depend on $X$ being an epoched \emph{Brownian} bridge\footnote{For example, all arguments here apply to $X_t = \sin(\pi t)$.}. Hence, we work without this specific assumption.

We consider Young differential equations of the form
\[dY_t = f_t(Y_t)\,dt + \si_t\,dX_t,\quad t\geq 0, Y_0 \in \R,\]
with $f_t : \R^d\to \R^d$ and $\si_t \in \R^{d\times m}$,
which is strictly speaking a different way of writing the integral equation
\begin{equation}
	\label{eq:linearRDE}
	Y_t = Y_0  + \int_0^t f_s(Y_s)\,ds + \int_0^t \si_s\,dX_s, \quad t\geq 0.
\end{equation}
Here, 
\[\int_0^t \si_s\,dX_s = \lim_{|\cP| \to 0} \sum_{[r,s]\in \cP} \si_r X_{r,s},\]
where the limit is taken with respect to all partitions of $[0,t]$ with mesh size $|\cP|$, and $X_{r,s} = X_s - X_r$. This is the \emph{Young integral}. If $X\in \hoelHZ \al([0,T])$ and $\si \in \hoelHZ \be([0,T])$ with $\al + \be > 1$, then the Young integral is guaranteed to exist (see Proposition \ref{prop:ynglve}).

To give an idea what is so special about (epoched) bridges consider the Young-Lóeve inequality.
\begin{proposition}[Young-Lóeve]
	\label{prop:ynglve}
	Let $\al,\be \in (0,1]$ with $\al + \be > 1$. Given $X\in \CHoelLoc{\al}$ and $\si\in \CHoelLoc \be$, the Young integral $\int_s^t \si_u\,dX_u$ exists, and we have
	\[\left|\int_s^t \si_u\,dX_u - \si_s X_{s,t}\right| \leq \frac{(t-s)^{\al+\be}}{2^{1-(\al + \be)}} \nrm{X}{\al;[s,t]} \nrm{\si}{\be;[s,t]}, \quad 0\leq s\leq t.\]
	Further, $\int_s^\blnk \si_u\,dX_u \in \CHoelLoc \al$.
\end{proposition}
\begin{proof}
	See \citet[Theorem 6.8]{Friz_Victoir_2010} and note that any $\al$-Hölder continuous function $X$ on $[s,t]$ (even if matrix-valued) has finite $1/\al$-variation $\pvar{X}{1/\al}$, with 
	\[\pvar{X}{1/\al} \leq (t-s)^\al\nrm{X}{\al}.\]
\end{proof}
Note that for any epoched bridge $X$ we have $X_{n,n+1} = 0$ for all $n\in \N_0$, so in this case Proposition \ref{prop:ynglve} implies
\begin{equation}
	\label{eq:ynglvebrdge}
	\left|\int_n^{n+1} \si_s\,dX_s\right| \leq \frac{1}{2^{1-(\al + \be)}} \nrm{X}{\al;[n,n+1]} \nrm{\si}{\be;[n,n+1]}, \quad n \in \N_0
\end{equation}
This is a crucial estimate in our convergence arguments (see the proof of Proposition \ref{prop:linearEpochedDecayYDE}).

\subsection{Existence and Uniqueness}
Our first aim is to show existence and uniqueness of a global solution $Y$ to \eqref{eq:linearRDE}.
\begin{proposition}
	\label{prop:YDEaddExUn}
	Suppose we are given the following.
	\begin{itemize}
		\item $\al, \be \in (0,1]$ with $\al + \be > 1$,
		\item $X : [0,\infty) \to \R^m \in \CHoelLoc\al$,
		\item $\si : [0,\infty) \to \R^{d\times m} \in \CHoelLoc\be$,
		\item $f : [0,\infty) \times\R^d \to \R^d$ is (jointly) measurable, such that
		\begin{enumerate}[(a)]
			\item $f_t(\blnk) \in \Lip$, uniformly in $t\geq 0$,
			\item $f_\blnk(0) \in \LpLoc1$.
		\end{enumerate}
	\end{itemize}
	Then there exists a unique solution $Y : [0,\infty) \to \R^d$ to the Young differential equation
	\begin{equation}
		\label{eq:YDEadd}
		dY_t = f_t(Y_t)\,dt + \si_t\,dX_t,\quad t\geq 0, Y_0 = y,
	\end{equation}
	and it satisfies $Y\in \CHoelLoc{(\al\wedge \be)^-}([0,\infty), \R^d)$.
\end{proposition}

\begin{proof}
	Let $T>0, \ga \in (0,\al \wedge \be)$ and define
	\[E = \set{Y \in \hoelHZ \ga([0,T], \R^d) : Y_0 = y}.\]
	This is a complete metric space when equipped with $d(Y,\tilde Y) = \nrm{Y - \tilde Y}{\ga}$.
	Define the map $\Ph : E\to E$ by
	\[(\Ph Y)_t = y_0 + \int_0^t f_s(Y_s)\,ds + \int_0^t \si_s \,dX_s.\]
	Note that the latter summand is a proper Young integral, since $\al + \be > 1$.
	We have
	\[|f_s(Y_s)|\leq |f_s(0)| + |f_s(Y_s) - f_s(0)| \leq |f_s(0)| + \nrm{f}{\Lip}|Y_s|,\]
	which is locally integrable in $s$. Thus, $\int_0^\blnk f_s(Y_s)\,ds \in \Lip([0,T])$.
	Further, $(\Ph Y)_0 = y_0$ and $\int_0^\blnk \si_s \,dX_s \in \hoelHZ \al([0,T]) \subseteq \hoelHZ \ga([0,T])$ by Proposition \ref{prop:ynglve}. Hence, $\Ph$ is well-defined.
	For $s,t\in [0,T]$ we estimate
	\begin{align*}
		|\Ph Y_{s,t} - \Ph \tilde Y_{s,t}| 	\leq& \int_s^t |f_r(Y_r) - f_r(\tilde Y_r)| \,dr \\
		\leq& \nrm{f}{\Lip} \int_s^t |Y_r - \tilde Y_r|\,dr \\
		\leq& \nrm{f}{\Lip} \nrm{Y - \tilde Y}{\ga} \int_s^t (r-s)^\ga\,dr\\
		\leq& \frac{1}{1+\ga}\nrm{f}{\Lip} \nrm{Y - \tilde Y}{\ga} (t-s)^{1+\ga}.
	\end{align*}
	Thus,
	\[|\Ph Y_{s,t} - \Ph \tilde Y_{s,t}|(t-s)^{-\ga} \leq \frac{T}{1+\ga}\nrm{f}{\Lip} \nrm{Y - \tilde Y}{\ga}, \quad s,t\in [0,T],\]
	i.e.
	\[\nrm{\Ph Y - \Ph \tilde Y}{\ga} \leq \frac{T}{1+\ga}\nrm{f}{\Lip} \nrm{Y - \tilde Y}{\ga},\]
	or, in other words, $\Ph$ is Lipschitz with constant bounded by $\frac{T}{1+\ga}\nrm{f}{\Lip}$.
	By picking $T = \frac{1+\ga}{2 \nrm{f}{\Lip}}$ we get $\nrm{\Ph}{\Lip} \leq \frac12$. In particular, $\Ph$ is a contraction and has a fixed point $Y\in E$, using the Banach fixed-point theorem. Being a fixed point means it is a solution of \eqref{eq:YDEadd} on $[0,T]$.
	If a solution $Y$ of \eqref{eq:YDEadd} exists on $[0,nT]$ for some $n\in \N$, then by applying the same argument with 
	\[E = \set{\tilde Y \in \hoelHZ \ga([nT,(n+1)T], \R^d) : \tilde Y_{nT} = Y_{nT}}\]
	extends the solution $Y$ to $[0,(n+1)T]$. Thus, a solution $Y$ exists on $[0,\infty)$.
	
	If there are two solutions $Y,\tilde Y$ on some interval $[0,T]$, then
	\[|Y_t - \tilde Y_t| \leq \int_0^t |f_s(Y_s) - f_s(\tilde Y_s)| \leq \nrm{f}{\Lip} \int_0^t |Y_s - \tilde Y_s|\,ds,\]
	and then Grönwalls inequality implies $Y_t = \tilde Y_t$, for all $t\in [0,T]$.
\end{proof}

\begin{proposition}
	\label{prop:varOfConst}
	Suppose we are given the following.
	\begin{itemize}
		\item $\al, \be \in (0,1]$ with $\al + \be > 1$,
		\item $X : [0,\infty) \to \R^m \in \CHoelLoc\al$,
		\item $\si : [0,\infty) \to \R^{d\times m} \in \CHoelLoc\be$,
		\item $A : [0,\infty) \to \R^{d\times d}\in \LpLoc 1 \cap L^\infty$,
		\item $b : [0,\infty) \to \R^d \in \LpLoc 1$.
	\end{itemize}
	Let $\ph$ be the unique solution to the linear matrix integral equation
	\begin{equation}
		\label{eq:linMatODE}
		\ph_t = 1_{d\times d} + \int_0^t A_s \ph_s\,ds.
	\end{equation}
	Then the unique solution $Y : [0,\infty) \to \R^d$ to the Young differential equation
	\begin{equation}
		\label{eq:YDElin}
		dY_t = A_t Y_t + b_t\,dt + \si_t\,dX_t, Y_0 \in \R^d.
	\end{equation}
	is given by
	\[Y_t = \ph_t\left(Y_0 + \int_0^t \ph_s^{-1} b_s\,ds + \int_0^t \ph_s^{-1} \si_s\,dX_s\right),\quad t\geq 0.\]
\end{proposition}
\begin{proof}
	Define
	\[Z_t = Y_0 + \int_0^t \ph_s^{-1} b_s\,ds + \int_0^t \ph_s^{-1} \si_s\,dX_s, \quad t\geq 0.\]
	Note that $\ph\in \CHoelLoc1$. Thus, the product formula (see \cite{friz2020course} Exercise 7.4) implies
	\begin{align*}
		\ph_t Z_t 	= & \ph_0 Z_0 + \int_0^t (d\ph_s) Z_t + \int_0^t \ph_s\,dZ_s \\
		= & \ph_0 Z_0 + \int_0^t A_s \ph_s Z_s\,ds + \int_0^t b_s\,ds + \int_0^t \si_s dX_s.
	\end{align*}
	Hence, $Y = \ph Z$ is a solution to \eqref{eq:linearRDE}. Uniqueness follows from Proposition \ref{prop:YDEaddExUn}.
\end{proof}
We can transform our main equation \eqref{eq:epochedCvx} into the simpler form (see Lemma \ref{lem:epochedCvxTransf} for details)
\[dY_t = -\tilde u_t \nabla \tilde \cR(Y_t)\,dt + \tilde u_t dX_t,\]
Here, $X$ is an epoched Brownian bridge, $\tilde u_t = (1 + tT)^{-\be}$ and $\tilde \cR$ is a random function satisfying the same conditions as $\cR$ in Theorem \ref{thm:main}, almost surely, except its global minimum is at $0$.
Thus, we will work mainly with equations of this form from now on.

\begin{docu}
	\begin{rem}
		Since $\si$ does not depend on $Y$ the present theory only requires the theory of Young integration and no rough path techniques. However, we can still apply various results from the theory of ($\alpha$-Hölder) rough paths, which is useful due to larger body of literature on this topic compared to the Young case. Note that any locally Lipschitz in time function $Z$ is trivially an $\rpX$-controlled rough path with remainder term $R_{s,t}^Z = Z_t - Z_s$ and Gubinelli derivative $Z' = 0$, since we have $\nrm{R_{s,t}^Z}{2\al} \lesssim \nrm{Z}{\Lip} < \infty$. Here, $X$ has been lifted to a rough path $\rpX$ in some way. Therefore, the rough integral against $\rpX$ coincides with the Young integral 
		\[\int_0^t Z_s\,d\rpX_s = \int_0^t Z_s\,dX_s\]
		and the specific lift $\rpX$ of $X$ plays no role.
		Further, the stipulation that $\al \in \left(\frac13, \frac12\right]$ is also unnecessary. All results in this section are true for any locally Hölder continuous path $X$.
	\end{rem}

	\paragraph{Periodic drivers}
	Let $T > 0$ and suppose $X$ has $T$-periodic increments, \tIe{} $X_{t+T} - X_{s+T} = X_t - X_s, s,t\geq 0$, and $X_0 = 0$.
	
	When considering the equations of the form \eqref{eq:linearRDE} it is usually not a restriction to assume that $X$ is $1$-periodic, that is $X_{t+1} = X_t, t\geq 0$. This is because if $Y$ satisfies \eqref{eq:linearRDE}, then for $\tilde Y_t := Y_{tT}$ and $\tilde X_t = X_{tT} - tX_T$ we have
	\begin{align*}
		\tilde Y_t = & Y_0 + T\int_0^t A_{sT} Y_{sT} + b_{sT}\,ds + \int_0^t \si_{sT} dX_{sT} \\
		= & \tilde Y_0 +  \int_0^T T A_{sT} \tilde Y_s + T b_{sT} + \si_{sT} X_T\,ds + \int_0^t \si_{sT} \,d\tilde X_t.
	\end{align*}
	
	\begin{lemma}
		Suppose $Y$ solves \eqref{eq:linearRDE}.
		Then $\tilde Y_t = Y_{tT}$ solves
		\begin{equation}
			\label{eq:linearRDE1period}
			d\tilde Y_t = \tilde A_t Y_t + \tilde b_t\,dt + \tilde \si_t\,d\tilde X_t, \quad \tilde Y_0 = Y_0.
		\end{equation}
		where
		\[\tilde A_t = T A_{tT}, \tilde b_t = Tb_{tT} + \si_{tT} X_T, \tilde \si_t = \si_{tT},\]
		and the driver $\tilde X_t = X_{tT} - tX_T$ is $1$-periodic.
		Conversely, if $\tilde Y$ solves \eqref{eq:linearRDE1period} for $1$-periodic $\tilde X$ and we are given $X_T\in \R^d$, then $Y_t := Y_{t/T}$ solves \eqref{eq:linearRDE}, where
		\[A_t = \frac1T\tilde A_{t/T}, b_t = \frac1T \tilde b_{t/T}- \tilde \si_t X_T, \si_t = \tilde \si_{t/T},\]
		with $X_t = \tilde X_{t/T} + \frac t T X_T$ having $T$-periodic increments.
	\end{lemma}

	Suppose now $X$ is $1$-periodic with $X_0 = 0$, and that $b = 0$. For every $Z\in \Lip_{\text{loc}}$ we have
	\begin{align*}
		\int_{n}^{n+1} Z_s \,dX_s = & \lim_{|\cP|\to 0} \sum_{[s,t] \in \cP} Z_{s+n} (X_{t+n} - X_{s+n}) = \lim_{|\cP|\to 0} \sum_{[s,t] \in \cP} Z_{s+n} (X_{t} - X_{s}) \\
		= & \int_0^1 Z_{s+n}\,dX_s.
	\end{align*}
	Thus, for $t\geq 0$,
	\begin{align*}
		\int_0^t Z_s\,dX_s = & \sum_{n=0}^{\floor t - 1} \int_n^{n+1} Z_s \,dX_s + \int_{\floor t}^t Z_s\,dX_s \\
		= & \sum_{n=0}^{\floor t - 1} \int_0^1 Z_{s+n} \,dX_s + \int_0^{\frk t} Z_{s + \floor t}\,dX_s.
	\end{align*}
	Hence, the unique solution to \eqref{eq:linearRDE} satisfies
	\begin{equation}
		\label{eq:linearRDEsol1Per}
		Y_t= \ph_t Y_0 + \sum_{n=0}^{\floor t-1}\ph_t \int_0^1  \ph_t^{s+n} \si_{s+n}\,dX_s + \int_0^{\frk t} \ph_t^{s+\floor t} \si_{s + \floor t}\,dX_s, \quad t\geq 0,
	\end{equation}
	where $\ph_t^s := \ph_t(\ph_s)^{-1}$, and
	\begin{equation}
		\label{eq:linearRDEsol1PerRec}	
		Y_t = \ph_{\frk t} Y_{\floor t} + \int_0^{\frk t} \ph_t^{s + \floor t} \si_{s+\floor t}\,dX_s, \quad t\geq 0.
	\end{equation}
\end{docu}

\begin{docu}
	\subsection{Ornstein-Uhlenbeck process with periodic noise}
	Let $\ka \in \R^{d\times d}$ be symmetric and positive definite, and set
	\[A_t = - \ka, \quad b_t  = 0,\quad \si_t = 1_{d\times d}, \quad t\geq 0.\]
	Then \eqref{eq:linearRDE} becomes
	\begin{equation}
		\label{eq:linearRDEc}
		dY_t = -\ka Y_t\,dt + dX_t, \quad Y_0 \in \R, t\geq 0,
	\end{equation}
	The solution $Y$ can be viewed as Ornstein-Uhlenbeck process driven by a $1$-periodic process $X$.

	Denote the unique solution of \eqref{eq:linearRDEc} with initial value $y\in \R^d$ by $(Y_t(y))_{t\geq 0}$.
	We are looking for an initial value $y\in \R^d$ is, such that $Y_1(y) = Y_0(y) = y$. We have
	\[0 = Y_0(y) - Y_1(y) = (1 - e^{-\ka})y  - \int_0^1 e^{-\ka (1 - s)}\,dX_s\]
	if and only if
	\[y = y^* := (1 - e^{-\ka})^{-1} \int_0^1 e^{-\ka (1 - s)}\,dX_s.\]
	Further,
	\begin{align*}
		Y_{n+1}(y^*) = & Y_n(y^*) e^{-\ka}+ \int_0^1 e^{-\ka (1-s)}\,dX_s\\
		= & Y_1(Y_n(x^*))
	\end{align*}
	Therefore,
	\[Y_n(y^*) = y^*, \quad n \in \N_0,\]
	and thus the solution to \eqref{eq:linearRDEc} with initial value $y^*$ (see \eqref{eq:linearRDEsol1PerRec})
	\[Y_t(y^*) = y^* e^{-\ka \frk t} + \int_0^{\frk t} e^{-\ka (\frk{t}-s)}\,dX_s\]
	is $1$-periodic, that is
	\[Y_{t+1}(y^*) = Y_t(y^*),\quad t\geq 0\]
	Thus, $Y(y^*)$ is a \emph{closed trajectory} of the dynamical system and its image
	\[O = \set{Y_t(y^*) : t \geq 0} = \set{Y_t(y^*) : t\in [0,1]}\]
	is a \emph{closed orbit}.
	
	\begin{proposition}
		For every $y\in \R^d$ we have
		\[d(Y_t(y), O) = \inf_{p\in O} |Y_t(y) - p| \leq |y-y^*| e^{-\la_{\min}(\ka) t}, t\geq 0.\]	
		In particular, $Y_t(y)$ converges to $O$, as $t\to \infty$, for every initial value $y\in \R^d$.
	\end{proposition}
	
	\begin{proof}
		The solution to 
		\[d(Y_t(y) - Y_t(\tilde y)) = - \ka (Y_t(y) - Y_t(\tilde y))\,dt\]
		is given by
		\[Y_t(y) - Y_t(\tilde y) = (y-\tilde y)e^{-\ka t}.\]
		Hence
		\[|Y_t(y) - Y_t(\tilde y)|  \leq |y-\tilde y| \specnrm{e^{-\ka t}} \leq |y-\tilde y| e^{-\la_{\min}(\ka) t}, \quad t\geq 0,\]
		and
		\[d(Y_t(y), O) \leq |Y_t(y) - Y_t(y^*)| \leq |y-y^*| e^{-\la_{\min}(\ka) t}, \quad t\geq 0.\]
	\end{proof}
\end{docu}

\subsection{Cooling down under epoched bridge noise}
\subsubsection{Preliminaries}
For some asymptotic integral estimates we use the theory of regular variation \citep[see][for more information]{Bingham_Goldie_Teugels_1987}.
A function $f : [0,\infty) \to (0,\infty)$ is called \emph{regularly varying of index} $\rho$ if $f$ is measurable and 
\[\lim_{t\to \infty} \frac{f(ct)}{f(t)} \to c^\rho, \quad c > 0.\]
Further, we call $f$ \emph{slowly varying} if it is regularly varying of index $\rho = 0$.
If $f$ is regularly varying, then  $f$ and $1/f$ are locally bounded and locally integrable on $[t_0,\infty)$ for some $t_0 \geq 0$. Moreover,
we can write 
\[f(t) = t^\rho \ell(t), \quad t > 0\]
where $\ell$ is slowly varying. 

If $f$ is regularly varying and $f\sim g$, then $g$ is also regularly varying with the same index. In particular, if $g = o(f)$ and $f$ is regularly varying of index $\rho$, then so is $f + g$ (provided $f+g > 0$ everywhere).

If $f$ is regularly varying with negative index, then $f(t) \to 0$, as $t\to \infty$.

If $\ell$ is slowly varying, then $\ell(t) = o(t^{\al}), t\to \infty$ for any $\al > 0$.
Examples of slowly varying functions include $\log(t)^{\al}$ for all $\al \in \R$.
\begin{lemma}
	\label{lem:exp-Ut}
	Let $\be \in (0,1)$ and $u$ be regularly varying with index $-\be$ and define $U_t = \int_0^t u_s\,ds$. Then 
	\[e^{-U_t} = o(f(t)), \quad t\to \infty,\]
	for any regularly varying function $f$.
\end{lemma}
\begin{proof}
	Writing $u_t = t^{-\be} \ell(t)$ for large $t$, we have by L'Hôpital's rule
	\[\lim_{t\to \infty} \frac{U_t}{\log t} = \lim_{t\to \infty}  t u_t = \lim_{t\to \infty} t^{1-\be} \ell(t) = \infty.\]
	Now, let $\al \in \R$. Then $-U_t + \al \log t \to -\infty$, and so $e^{-U_t} t^\al \to 0$, as $t\to \infty$.
	If $f$ is regularly varying of index $\al$, then $e^{-U_t} = o(t^{-|\al| - 1}) = o(f(t))$ as $t\to \infty$.
\end{proof}

\begin{proposition}
	\label{prop:laplaceIntegralBnd}
	Let $f$ and $u$ be regularly varying functions with indices $-\rho, -\beta < 0$ and $\beta < 1$. Suppose further that $f$ is locally bounded and $u\in \LpLoc{1}$ is non-increasing.
	Then we have
	\[\int_0^t f(s)e^{-U_t^s}\,ds \leq \frac{f(t)}{u(t)} + o\left(\frac{f(t)}{u(t)}\right),\quad t\to \infty,\]
	where $U_t^s = \int_s^t u(s)\,ds$.
\end{proposition}
\begin{proof}
	Since $u$ is non-increasing, $U$ is concave and we have
	\[U(s) \leq U(t) + u(t)(s-t), \quad s,t\geq 0,\]
	where $U(t) = U_t^0$.
	Therefore,
	\begin{equation}
		\label{eq:laplaceIntegralBndInitEst}
		\int_0^t f(s)e^{-U_t^s}\,ds \leq \int_0^t f(s)e^{-(t-s)u(t)}\,ds = f(t)\int_0^t \frac{f(t-s)}{f(t)} e^{-su(t)}\,ds.
	\end{equation}
	Let $\tau : [0,\infty) \to [0,\infty)$ be non-increasing, such that 
	\begin{equation}
		\label{eq:timeTauCond}
		\frac{\tau_t}{t}\to 0,\,\tau_t u(t) \to \infty, \quad t\to \infty.
	\end{equation}
	In particular, $\tau_t\to \infty$ since $u(t)\leq u(0), t\geq 0$. We make a particular choice of $\tau$ towards the end.
	We split the integral on the RHS of Inequality \eqref{eq:laplaceIntegralBndInitEst} into a main part $\int_0^{\tau_t}\dots\,ds$ and a tail part $\int_{\tau_t}^t\dots\,ds$.
	
	Let us first estimate the main part. Because $f$ is regularly varying with index $-\rho$, we have 
	\[\lim_{t\to \infty} \sup_{c\in [a,\infty)} \left|\frac{f(ct)}{f(t)} - c^{-\rho}\right| = 0,\]
	for all $a > 0$ \citep[Theorem 1.5.2]{Bingham_Goldie_Teugels_1987}.
	Since $t-s = t(1-s/t)$ we have
	\begin{align*}
		\sup_{s\in (0,\tau_t]} \left|\frac{f(t-s)}{f(t)} - 1\right| = & \sup_{c\in [1-\frac{\tau_t}{t}, 1)} \left|\frac{f(ct)}{f(t)} - 1\right| \\
		\leq & \sup_{c\in [1-\frac{\tau_t}{t}, 1)} \left|\frac{f(ct)}{f(t)} - c^{-\rho}\right| + \sup_{c\in [1-\frac{\tau_t}{t}, 1)} |c^{-\rho} - 1| \\
		\to & 0,
	\end{align*}
	because $\frac{\tau_t}{t}\to 0$, as $t \to \infty$.
	Hence,
	\[\int_0^{\tau_t} \frac{f(t-s)}{f(t)} e^{-su(t)}\,ds \sim \int_0^{\tau_t} e^{-su(t)}\,ds = \frac{1}{u(t)} (1 - e^{-\tau_t u(t)}) \sim \frac{1}{u(t)}\]
	as $t\to \infty$.
	
	To estimate the tail integral let $\ep > 0$.
	By Potter's theorem \citep[Theorem 1.5.6 (iii)]{Bingham_Goldie_Teugels_1987}, there exists a $t_0 \geq 0$ with
	\[\frac{f(r)}{f(t)} \lesssim \left(\left(\frac r t\right)^{-\rho + \ep} \vee \left(\frac r t\right)^{-\rho - \ep}\right) = \left(\frac t r\right)^{\rho + \ep} \leq t_0^{-(\rho + \ep)}t^{\rho + \ep},\]
	uniformly over $t\geq r\geq t_0$. In particular, by writing $r = t - s$ we have
	\[\sup_{s\in [0,t-t_0]}\frac{f(t-s)}{f(t)} \lesssim t^{\rho + \ep},\]
	uniformly over large $t$.
	Since $f$ is locally bounded, we have
	\[\sup_{s\in [t-t_0,t]}\frac{f(t-s)}{f(t)} \lesssim \frac1{f(t)} \sim \ell(t)t^{\rho}, \quad t\to \infty,\]
	for some slowly varying function $\ell$. Hence,
	\[\sup_{s\in [0, t]}\frac{f(t-s)}{f(t)} \lesssim t^{\rho + \ep}\ell(t),\]
	uniformly over large $t$, for slowly varying $\ell$.
	Thus,
	\[\int_{\tau_t}^t \frac{f(t-s)}{f(t)} e^{-su(t)}\,ds \lesssim \ell(t)t^{\rho + \ep}\int_{\tau_t}^\infty e^{-s u_t}\,ds = \frac{1}{u(t)} \ell(t)t^{\rho +\ep} e^{-\tau_t u(t)},\]
	uniformly over large $t$.
	Finally, define $\tau_t = \frac{(\rho + 2\ep) \log t}{u(t)}$. Then the first convergence in \eqref{eq:timeTauCond} is satisfied because $u$ is regularly varying with index $-\be \in (-1,0)$. The second follows from $\log t \to \infty$, as $t\to \infty$.
	Moreover, $t^{\rho+\ep}e^{-\tau_t u(t)} = t^{-\ep}$ and so 
	\[\int_{\tau_t}^t \frac{f(t-s)}{f(t)} e^{-su(t)}\,ds = o\left(\frac{1}{u(t)}\right), \quad t\to \infty.\]
	Using Inequality \eqref{eq:laplaceIntegralBndInitEst} we conclude
	\[\int_0^t f(s)e^{-U_t^s}\,ds \leq \frac{f(t)}{u(t)} + o\left(\frac{f(t)}{u(t)}\right),\quad t\to \infty.\]
\end{proof}

\begin{lemma}
	\label{lem:sumVsInt}
	Let $a,b\in \N_0$ with $a < b$ and $f : [a,b]\to \R$ be integrable with finite $1$-variation $\pvar f1$. Then
	\[\left|\sum_{n = a+1}^b f(n) - \int_a^b f(t)\,dt\right| \leq \pvar f1.\]
\end{lemma}
\begin{proof}
	We calculate 
	\begin{align*}
		\sum_{n = a+1}^b f(n) 	= &\sum_{n=a}^{b-1} f(n+1) \\
		= &\sum_{n=a}^{b-1} \int_n^{n+1} f(t)\,dt + \sum_{n=a}^{b-1} \left(f(n+1)-\int_n^{n+1} f(t)\,dt\right)
	\end{align*}
	Note that
	\[\left|f(n+1)-\int_n^{n+1} f(t)\,dt\right| \leq \sup_{t\in [n,n+1)} |f(t) - f(n+1)|.\]
	Let $\ep > 0$. There exist $t_a,\dots, t_{b-1}$ with $t_n \in [n,n+1)$, such that
	\[\sup_{t\in [n,n+1)} |f(t) - f(n+1)| \leq |f(t_n) - f(n+1)| + \ep.\]
	Then
	\begin{align*}
		\left|\sum_{n=a}^{b-1} \left(f(n+1)-\int_n^{n+1} f(t)\,dt\right)\right| \leq \pvar f1 + (b-a)\ep.
	\end{align*}
	Since $\ep > 0$ was arbitrary, the desired conclusion follows.
\end{proof}

Now, let $\be \in (0,1), c > 0$ and consider $u : [0,\infty) \to [0,1], t\mapsto \frac{1}{(1+ct)^\be}$.
Given a positive definite and symmetric matrix $\ka$, the unique solution to the ODE
\[\dot \ph_t^s = -u_t \ka \ph_t^s\, \quad t\geq s, y_s = 1_{d\times d}\]
is given by $\ph_t^s = e^{- \ka U_t^s}$, where $U_t^s = \int_s^t u_r\,dr$, and we have
\begin{equation}
	\label{eq:linODEest}
	\specnrm{\ph_t^s} = \la_{\max}(\ph_t^s) \leq e^{-\la U_t^s},
\end{equation}
where $\la := \la_{\min}(\ka)$. In particular, $\ph_t^s$ converges to $0$, as $t\to \infty$.

\begin{lemma}
	\label{lem:lrScheduleEst}
	We have
	\begin{enumerate}[(a)]
		\item $u \in \Lip^1([0,\infty))$,
		\item $u$ is strictly decreasing, convex and $\lim_{t\to \infty} u_t = 0$,
		\item $U$ is concave and $\lim_{t\to \infty} U_t = \infty$,
		\item $|\dot u_t| = c\be u_t^{2 + \ga}$ for all $t\geq 0$, where $\ga = \frac{1-\be}{\be} > 0$,
		\item \[\nrm{u_{\cdot} \ph^{\cdot}_t}{\Lip;[k,(k+1)\wedge t]} \leq (\la_{\max}(\ka)  + c \be u_k^\ga) u_k^2 e^{-\la U_t^{(k+1)\wedge t}},\]
		for all $t\geq 1$ and $k\leq t$,
		In particular, $\nrm{u_{\cdot} \ph^{\cdot}_t}{\Lip;[k,(k+1)\wedge t]} = o(u_t), t\to \infty$.
		\item For all $\rho > 1$ and $t\geq 1$ we have
		\[\sum_{k=0}^{\floor t-1} u_k^{\rho} e^{-\la U^{k+1}_t} \leq I_t(\rho) + I_t(\rho + 1) + \rho c\be I_t(\rho + \ga + 1) + e^{-\la U_t},\]
		where $I_t(\al) = \int_0^{\floor t-1} u_s^\al e^{-\la U^{s+1}_t}\,ds$.
		\item $I_t(\rho) \leq \la^{-1} (ct)^{-(\be(\rho - 1))} + o(t^{-(\be(\rho - 1))}), t\to \infty$, for all $\rho > 1$.
		\item $e^{-\la U_t} = o(t^{-\al}), t\to \infty$, for all $\al > 0$.
		\item \[\sum_{k=0}^{\floor t-1} \nrm{u_\cdot \ph^\cdot_t}{\Lip;[k,k+1]} \leq \condNr{\ka} (ct)^{-\be} + o(t^{-\be}),\]
		as $t\to \infty$.
	\end{enumerate}
\end{lemma}

\begin{docu}
	$e^{-\la U_t} = o(t^{-\al}), t\to \infty$, for all $\al > 0$ implies $e^{-\la U_t} = o(\ell(t))$ for some slowly varying function $\ell$ by Theorem 2.3.6. Binghzam (a result due to Vuilleumier).
\end{docu}

\begin{proof}
	\begin{enumerate}[(a)]
		\item $u$ is differentiable with $\dot u_t = -c \be(1+t)^{-(1+\be)}$ and $|\dot u_t| \leq \be$,
		\item Straightforward.
		\item We have 
		\[U_t = \frac{1}{1-\be}\left((1+t)^{1-\be} - 1\right),\]
		so $\lim_{t\to \infty} U_t = \infty$. Concavity follows from $u$ being strictly decreasing.
		\item $|\dot u_t| = c \be (1+t)^{-(1+\be)}= c \be (1+t)^{-(1-\be)}(1+t)^{-2\be} = c \be u_t^{2+\ga}$ for all $t\geq 0$,
		\item Let $f_s = u_s \ph_t^s$. Then 
		\[\dot f_s = (\dot u_s 1_{d\times d} + u_s^2 \ka)\ph_t^s,\]
		and so 
		\[\specnrm{\dot f_s} \leq \specnrm{\dot u_s 1_{d\times d} + u_s^2 \ka}\specnrm{\ph_t^s}\leq (|\dot u_s| + u_s^2 \specnrm \ka)e^{-\la U_t^s} =(\specnrm \ka + c \be u_s^\ga)  u_s^2 e^{-\la U_t^s},\]
		for all $0\leq s\leq t$. Taking the supremum over $[k,k+1]$ for each factor individually yields the estimate.
		\item Set $n = \floor t$. By applying Lemma \ref{lem:sumVsInt} we have
		\[e^{-\la U_t} \sum_{k=0}^{n-1} u_k^{\rho} e^{\la U_{k+1}}\leq e^{-\la U_t} \pvar{(u^\rho e^{\la U_{\blnk + 1}})|_{[0,n-1]}}1 + e^{-\la U_t} + I_t(\rho).\]
		Since
		\[|\der_s (u^\rho_s e^{\la U_{s+1}})| = (\rho u_s^{\rho - 1}|\dot u_s|+ u_s^{\rho + 1})e^{\la U_{s+1}} \leq u_s^{\rho + 1}(1 + \rho c \be u_s^\ga) e^{\la U_{s+1}},\]
		we conclude
		\[e^{-\la U_t} \pvar{(u^\rho e^{\la U_{\blnk + 1}})|_{[0,n-1]}}1 \leq I_t(\rho + 1) + \rho c \be I_t(\rho + \ga + 1).\]
		\item Proposition \ref{prop:laplaceIntegralBnd} implies
		\[I_t(\rho) \leq \int_1^t u_{s-1}^\rho e^{-\la U_t^s} \leq \frac{u_{t-1}^\rho}{\la u_t} + o\left(\frac{u_{t-1}^\rho}{u_t}\right), \quad t\to \infty.\]
		Now observe that for $c = 1$
		\[\frac{u_{t-1}^\rho}{u_t} = u_{t-1}^{\rho-1}\left(1 + \frac{1}{t}\right)^\be = t^{-(\be(\rho - 1))} + o(t^{-(\be(\rho - 1))}), t\to \infty,\]
		so for general $c > 0$
		\[\frac{u_{t-1}^\rho}{u_t} = (ct)^{-(\be(\rho - 1))} + o(t^{-(\be(\rho - 1))}), t\to \infty.\]
		\item Follows from Lemma \ref{lem:exp-Ut}.
		\item By applying (e) and (f) we have
		\begin{align*}
			\sum_{k=0}^{n-1} \nrm{u_\cdot \ph^\cdot_t}{\Lip;[k,k+1]} \leq & \sum_{k=0}^{n-1} u_k^2(\la_{\max}(\ka) +  \be u_k^\ga) e^{-\la U_t^{(k+1)}} \\
			\leq & \la_{\max}(\ka)(I_t(2) + I_t(3) + 2 c \be I_t(3 + \ga) + e^{-\la U_t}) \\
			&+ \be(I_t(2+\ga) + I_t(3+\ga) + (2+\ga)c \be I_t(3 + 2\ga) + e^{-\la U_t}).
		\end{align*}
		We conclude the desired result using (g) and (h).
	\end{enumerate}
\end{proof}

\subsubsection{Convergence results}

\begin{proposition}
	\label{prop:linearEpochedDecayYDE}
	Let $X$ be a locally $\al$-Hölder epoched bridge and $Y$ be the solution to the linear \tYDE{}
	\[dY_t = - u_t\ka Y_t\,dt + u_t\,dX_t, \quad Y_0 \in \R, t\geq 0.\]
	Then
	\[|Y_t|\leq \left(\frac{1}{1-2^{-\al}}\condNr{\ka} + 1\right)c^{-\be}\frac{x^*_t}{t^{\be}} + o\left(x^*_t t^{-\be}\right), \quad t\to \infty,\]
	where $x^*_t := \max_{k\leq t} \nrm{X}{\alpha;[k,(k+1)\wedge t]}$.
\end{proposition}

\begin{proof}
	Let $t\geq 0$ and $n = \floor t$. By Proposition \ref{prop:varOfConst} we have
	\[Y_t = \ph_t Y_0 + \int_n^t u_s \ph^s_t\,dX_s + \sum_{k=0}^{n-1} \int_0^1 u_{s+k}\ph^{s+k}_t \,dX_{s+k}, \quad n \in \N.\]
	We estimate using the Young-Lóeve inequality in its original form (Proposition \ref{prop:ynglve}) and in the form \eqref{eq:ynglvebrdge} (with $\be = 1$), as well as Inequality \eqref{eq:linODEest}
	\begin{align*}
		|Y_t| \leq & |Y_0| e^{-\la U_t} + (|u_n \ph_t^n X_{n,t}| + C\nrm{u_{\cdot} \ph^{\cdot}_t}{\Lip;[n,t]} \nrm{X}{\alpha;[n,t]}) + C\sum_{k=0}^{n-1} \nrm{u_{\cdot} \ph^{\cdot}_t}{\Lip;[k,k+1]} \nrm{X}{\alpha;[k,k+1]},
	\end{align*}
	where $C = \frac{1}{1-2^{-\al}}$.
	We have $e^{-\la U_t} = o(t^{-\be})$ by Lemma \ref{lem:lrScheduleEst} (h).
	Further,
	\[|u_n \ph_t^n X_{n,t}| \leq u_n \specnrm{\ph_t^n}|X_{n,t}| \leq u_n\cdot 1 \cdot (t-n)^\al \nrm{X}{\al;[n,t]} =  (x^*_tt^{-\be} + o(x^*_tt^{-\be})) ,\]
	$t\to \infty$, and
	\[\nrm{u_{\cdot} \ph^{\cdot}_t}{\Lip;[n,t]} \nrm{X}{\alpha;[n,t]} = o(x_t^*t^{-\be}),\quad t\to \infty,\]
	by Lemma \ref{lem:lrScheduleEst} (e).
	Finally,
	\[\sum_{k=0}^{n-1} \nrm{u_{\cdot} \ph^{\cdot}_t}{\Lip;[k,k+1]} \nrm{X}{\alpha;[k,k+1]} \leq  \condNr{\ka}\frac{x_t^*}{t^{\be}} + o(x_t^* t^{-\be}),\quad  t\to \infty, \]
	by Lemma \ref{lem:lrScheduleEst} (i).
\end{proof}

\begin{proposition}
	\label{prop:linToCvxDecay}
	Let $\cR : \R^d \to \R \in \dC 2$ be $\la$-strongly convex and $L$-smooth with $\nabla \cR(0) = 0$ and $\nabla^2 \cR$ Hölder continuous.
	Let $X$ be locally Hölder continuous and assume that $X$ does not vanish on any closed interval of positive measure.
	Let $Y_0 = Z_0 \in \R^d$, and $Y, Z$ be the solutions to the Young differential equations
	\begin{align*}
		dY_t = &-u_t \nabla \cR(Y_t)\,dt + u_t\,dX_t,\\
		dZ_t = &- u_t \nabla^2 \cR(0) Z_t\,dt + u_t\,dX_t, \quad t\geq 0.
	\end{align*}
	Let $f$ be regularly varying with negative index and assume $|Z_t| \leq f(t), t\to \infty$.
	Then also
	\[|Y_t| \leq f(t) + o(f(t)), \quad t\to \infty.\]
\end{proposition}

\begin{docu}
	If $\nabla^2 \cR$ is not globally Hölder continuous, then we still can show that $Y$ converges, except with
	\[|Y_t|\leq (2L\vee 1)f(t) + o(f(t)).\]
	This is because $|r(y)|\leq 2 L|y|$ and then the penultimate estimate yields $|\delt_t|\leq \leq 2Lf(t) + o(f(t))$.
\end{docu}

\begin{proof}
	Firstly, assume $\cR$ is not quadratic. Otherwise, $Y = Z$ and we are done.
	Now, using Hadarmard's lemma we have
	\[r(y) := \nabla \cR(y) - \nabla^2 \cR(0) y = \int_0^1 (\nabla^2 \cR(ty) - \nabla^2 \cR(0))y\,dt.\]
	Thus, the Hölder continuity of $\nabla^2 \cR$ implies
	\[|\nabla^2 \cR(ty) - \nabla^2 \cR(0)|\lesssim |ty|^\ga \leq |y|^\ga, \quad t\in [0,1], y\in \R^d,\]
	for some $\ga\in (0,1]$. Thus,
	\begin{equation}
		\label{eq:cvxRemEst}
		|r(y)|\lesssim |y|^{1+\ga}
	\end{equation}
	uniformly over $y\in \R^d$, and we can write
	\[dY_t = -u_t (\ka Y_t + r(Y_t))\,dt + u_t\,dX_t, \quad t\geq 0,\]
	where $\ka := \nabla^2 \cR(0)$.
	Let $\delta = Y - Z$. Then
	\[\dot \delt_t = - u_t \ka \delt_t - u_t r(Y_t).\]
	Furthermore,
	\begin{align*}
		\frac12\der_t(|\delt_t|^2) 	= \frac12 \der_t \innp{\delt_t}{\delt_t} = \innp{\dot \delt_t}{\delt_t} = &-u_t \innp{\ka \delt_t + r(Y_t)}{\delt_t} \\
		= &-u_t \innp{\ka \delt_t + r(Y_t) - r(Z_t)}{\delt_t} + u_t \innp {r(Z_t)}{\delt_t}, \quad t\geq 0.		
	\end{align*}
	Since $\cR$ is $\la$-strongly convex we have
	\[\innp{\ka y + r(y) - (\ka z + r(z))}{y-z} = \innp{\nabla \cR(y) - \nabla \cR(z)}{y-z} \geq \la |y-z|^2, \quad y,z\in \R^d.\]
	Hence, writing $v = |\delt|$,
	\[\dot v_t v_t = \frac12\der_t(v_t^2) \leq - u_t \la v_t^2 + u_t |r(Z_t)|v_t,\]
	and so
	\begin{equation}
		\label{eq:|delta|ineq}
		\dot v_t \leq -u_t \la v_t + u_t |r(Z_t)|,
	\end{equation}
	for all $t\geq 0$, such that $\delt_t \neq 0$.
	The set
	\[\set{t \geq 0 : \delt_t = 0}\]
	has Lebesgue measure zero. To show this note that if $\delta_t = 0$, then
	\[\dot \delt_t = - u_t r(Y_t).\]
	Assume $\delt = 0$ on an interval $[t,w]$. Then 
	\[\dot \delt_s = - u_s r(Y_s) = 0,\quad s\in [t,w].\] 
	Since $\cR$ is not quadratic we have $r(y) = 0$ if and only if $y = 0$.
	Together with $u > 0$ everywhere this implies $Y = 0$ on $[t,w]$.
	Thus,
	\[Y_s = Y_t + \int_t^s u_v \,dX_v = \int_t^s u_v \,dX_v\]
	implying $X = 0$ on $[t,w]$, which we assumed to be impossible.
	Thus, $\delt_t = 0$ only at isolated points $t\geq 0$.
	Hence, the set of $\delt$s zeros has measure $0$.
	
	Moving on, define the integrating factor $I_t = e^{\la U_t}$. Then using Inequality \eqref{eq:|delta|ineq}
	\[\der_t (I_tv_t) = I_t \dot v_t + \la u_t v_t I_t \leq u_t |r(Z_t)| I_t,\]
	for almost all $t\geq 0$.
	Hence,
	\[|\delt_t|e^{\la U_t} = I_t v_t \leq \int_0^t u_s|r(Z_s)|e^{\la U_s}\,ds.\]
	Note that the function $\tilde f = uf^{1+\ga}$ is again regularly varying with negative index.
	Thus, using Inequality \eqref{eq:cvxRemEst} and Proposition \ref{prop:laplaceIntegralBnd} for the function $\tilde f$,
	\[|\delt_t| \leq \int_0^t u_s e^{-\la U_t^s} |Z_s|^{1+\ga}\,ds \leq \int_0^t u_s e^{-\la U_t^s} f(s)^{1+\ga}\,ds = O\left(\frac{\tilde f(t)}{u(t)}\right) = o(f(t)), \quad t\to \infty.\]
	We conclude
	\[|Y_t| \leq |\delt_t| + |Z_t| \leq f(t) + o(f(t)), \quad t\to \infty.\]
\end{proof}

\begin{corollary}
	\label{cor:epochedCvxSimplConv}
	Let $X$ be a locally $\al$-Hölder epoched bridge that does not vanish on any closed interval of positive measure, and such that
	\[\max_{k\leq t} \nrm{X}{\alpha;[k,(k+1)\wedge t]} \leq \ell(t), \quad t\to \infty,\]
	for some slowly varying function $\ell$.
	Further, let $\cR : \R^d \to \R \in \dC 2$ be $\la$-strongly convex and $L$-smooth with $\nabla \cR(0) = 0$ and $\nabla^2 \cR$ Hölder continuous.
	If $Y$ is the solution to the \tYDE{}
	\[dY_t = - u_t \nabla \cR(Y_t)\,dt + u_t\,dX_t, \quad Y_0 \in \R, t\geq 0,\]
	then
	\[|Y_t|\leq \left(\frac{1}{1-2^{-\al}} \frac{L}{\la} + 1\right)c^{-\be}\frac{\ell(t)}{t^{\be}} + o\left(\ell(t) t^{-\be}\right), \quad t\to \infty.\]
\end{corollary}
\begin{proof}
	We apply Proposition \ref{prop:linearEpochedDecayYDE} to the linear ODE
	\[dZ_t = - u_t \nabla^2 \cR(0) Z_t\,dt + u_t\,dX_t.\]
	Then, Proposition \ref{prop:linToCvxDecay} implies the desired conclusion.
\end{proof}

\begin{docu}
	Stefan 22.06.2025
	\begin{itemize}
		\item I tried to make the conditions on $u$ generic initially, but I just ended up having to add more and more conditions to get easy to understand convergence rate, and its not clear what functions even satisfy all these conditions. So in the end I settled on $(1+t)^{-\be}$. Probably most LR schedules that could theoretically work arent even that interesting in practice anyway. Comment 12.07.: regularly varying $u$ works surely, but then I have to state everything asymptotically from the get go and I didnt want that in case I want to derive less asymptotic results too. And frankly whos gonna use an LR schedule with extra logs.
		\item There is no way the convergence rate I have deduced is tight, at least not for certain cases like single shuffle or RR. For single shuffle or any periodic $X$ we can leave out $\sqrt{\log t}$ entirely. Moreover, RR should be even better than SS according to SGDo literature. Even without the $\sqrt{\log t}$ the $2\be - 1$ should be improvable to $\be$ (see Gürbüzbalaban2019). - Comment: 30.06.  fixed, but I dont think I can prove a non-asymptotic rate. - Comment 12.07.: Gürbüzbalabans rates are also asymptotic, I just misunderstood their result. 
		\item It is highly doubtful that you can deduce convergence under the general Robbins-Monroe conditions this way, because of the $\sqrt{\log t}$. Maybe for periodic driver it could work?
		\item For $u_t = \frac{1}{\log(1+t)}$ or $u_t = 1/(1+t)$. What happens then?
	\end{itemize}
\end{docu}

\section{Proof of the main theorem}
Firstly, let us prove that $(\nabla \cR)^{-1}$ is actually well-defined.
\begin{lemma}
	Let $\la > 0$. Suppose $\cR$ is $\la$-strongly convex with Lipschitz gradient. Then $\nabla \cR : \R^d\to \R^d$ is bijective.
\end{lemma}
\begin{proof}
	Strong convexity implies strong monotonicity, that is
	\[\innp{\nabla \cR(x) - \nabla \cR(y)}{x-y} \geq \la |x-y|^2, \quad x,y\in \R^d.\]
	In particular, $\nabla \cR$ is injective.
	To show surjectivity we use the Browder-Minty theorem \citep[see][Theorem 10.49]{renardy2006introduction}, identifying $\R^d$ with its dual space. Indeed, $\nabla \cR$ is monotone, as shown before. Also since $\nabla \cR$ is Lipschitz, it is in particular continuous and preserves bounded sets. To show coercivity, note that strong convexity of $\cR$ implies
	\[\cR(0) \geq \cR(x) + \innp{\nabla \cR(x)}{0-x} + \frac{\la}{2}|x|^2, \quad x\in \R^d.\]
	That is,
	\[\innp{\nabla \cR(x)}{x} \geq \cR(x) - \cR(0) + \frac{\la}{2} |x|^2.\]
	In particular,
	\[\lim_{x\to 0}\frac{\innp{\nabla \cR(x)}{x}}{|x|} = \infty.\]
	Hence, $\nabla \cR$ is coercive, and thus also surjective.
\end{proof}
Now, let us transform equation \eqref{eq:epochedCvx} into a simpler form.
We can rewrite
\[dY_t = - u_t (\nabla \cR(Y_t) - T^{-1/2}\si Z)\,dt + u_t \sqrt T \si dX_{t/T},\]
or equivalently
\[dY_{tT} = - u_{tT}\nabla \hat \cR(Y_{tT})\,dt + u_{tT} \sqrt T \si dX_t,\]
where $Z = \frac{1}{\sqrt T}\hat W_T \sim \cN(0,1_{d\times d}), \hat W_t = \sqrt T X_{t/T} + \frac t {\sqrt T} Z$ and $X$ is an epoched Brownian bridge independent of $Z$,
and $\hat \cR(y) =  \cR(y) - T^{-1/2} \si Z y$. Note that 
\[(\nabla \hat \cR)^{-1}(0) = (\nabla \cR - T^{-1/2}\si Z)^{-1}(0) = (\nabla \cR)^{-1}(T^{-1/2}\si Z).\]
Define 
\[\tilde Y_t = \frac{1}{\sqrt T} \si^{-1}(Y_{tT} - (\nabla \hat \cR)^{-1}(0)), \quad t\geq 0.\]
Then
\[d\tilde Y_t = - u_{tT} \frac{1}{\sqrt T}\si^{-1} \nabla \hat \cR(\sqrt T \si \tilde Y_t + (\nabla \hat \cR)^{-1}(0))\,dt + u_{tT}dX_t, \quad t\geq 0.\]
Equivalently, we can write
\[d\tilde Y_t = - u_{tT} \nabla \tilde \cR(Y_t)\,dt + u_{tT}\,dX_t,\]
where
\begin{align*}
	\tilde \cR(y) :=& T^{-1}\si^{-2} \hat \cR(\sqrt T \si y + (\nabla \hat \cR)^{-1}(0)) \\
	= & T^{-1}\si^{-2} \cR(\sqrt T \si y + T^{-1}\si \hat W_T) - T^{-1}\si \hat W_T y, \quad y\in \R^d.
\end{align*}
Let us summarize this procedure in a proposition.
\begin{lemma}
	\label{lem:epochedCvxTransf}
	Let $Y$ be the solution to \eqref{eq:epochedCvx}. Then 
	\[\tilde Y_t = \frac{1}{\sqrt T}\si^{-1}(Y_{tT} - (\nabla \cR)^{-1}(T^{-1}\si \hat W_T))\]
	is the unique solution to the \tYDE{}
	\[d\tilde Y_t = - \tilde u_t \nabla \hat \cR(\tilde Y_t) + \tilde u_t\,dX_t, \quad t\geq 0,\]
	where $\tilde u_t = u_{tT}$ and
	\[\tilde \cR(y) =  T^{-1}\si^{-2} \cR(\sqrt T \si y + T^{-1}\si \hat W_T) - T^{-1}\si \hat W_T y, \quad y\in \R^d.\]
\end{lemma}

\begin{proof}[Proof of Theorems \ref{thm:main} and \ref{thm:mainAlt}]
	Recall the definition of $Y$ in \eqref{eq:epochedCvx}. Apply Lemma \ref{lem:epochedCvxTransf}, then
	\[Y_t = \sqrt T \si \tilde Y_{t/T} + (\nabla \cR)^{-1}(T^{-1}\si \hat W_T).\]
	Note that $X$ does not vanish on any closed interval of positive measure, almost surely.
	Suppose for now we are given slowly varying function $\ell$ with
	\begin{equation}
		\label{eq:EBMdriverSlow}
		\max_{k\leq t} \nrm{X}{\alpha;[k,(k+1)\wedge t]} \leq \ell(t), \quad a.s.,t\to \infty.
	\end{equation}
	By Corollary \ref{cor:epochedCvxSimplConv}
	\[\left|Y_t - (\nabla \cR)^{-1}(T^{-1}\si \hat W_T)\right| \leq \sqrt T |\si| \left(\frac{1}{1-2^{-\al}}\frac{L}{\la} + 1\right)(cT)^{-\be}\frac{\ell(t)}{t^{\be}} + o\left(\ell(t) t^{-\be}\right), \quad t\to \infty.\]
	Here, we used that $\nabla^2 \tilde \cR(0) = \nabla^2 \cR( (\nabla \cR)^{-1}(T^{-1}\si \hat W_T))$.
	
	We can find a slowly varying function $\ell$ such that Inequality \eqref{eq:EBMdriverSlow} holds true. Indeed, by Lemma \ref{lem:epochedBBgrowth} we can set 
	\[\ell(t) :=  a^{-1/2}\sqrt{\log t} + g(t) \geq a^{-1/2}\sqrt{\log{(\floor t + 1)}},\] 
	for $a \in (0, \frac{1}{2(1-b)b^{1-2\al}})$, where $b = \frac{1-2\al}{2-2\al}$, and
	\[g(t) = a^{-1/2}(\sqrt{\log{(\floor t + 1)}} - \sqrt{\log t})   = o(\sqrt{\log t}), \quad t\to\infty.\]
	If we pick $\al = 0.42$, $a = 0.8 \in (0, 0.858581) = (0, \frac{1}{2(1-b)b^{1-2\al}})$, then
	\[a^{-1/2} = 1.11803 < 1.2, \quad a^{-1/2} \frac{1}{1-2^{-\al}} = 4.61727 < 4.7,\]
	proving Theorem \ref{thm:main} (the second constant cannot be lowered much further).
	Assume now there exists a number $J\in \N$, such that $\cI := \set{(W_{(j+t)T}-W_{jT})_{t\in [0,1]} : j\in \N}|$ satisfies $|\cI| = J$, almost surely. Then we can instead set $\ell(t) = \max_{w\in \cI} \nrm{w}{\al}, t\geq 0$ in Inequality \eqref{eq:EBMdriverSlow}, proving Theorem \ref{thm:mainAlt}.
\end{proof}

\acks{I thank Stefan Ankirchner for our joint work \citet{perko2022towards}, which shaped the conceptual motivation for this article.}

\bibliographystyle{abbrvnat}
\bibliography{sgdobib}

\begin{thebibliography}{22}
\providecommand{\natexlab}[1]{#1}
\providecommand{\url}[1]{\texttt{#1}}
\expandafter\ifx\csname urlstyle\endcsname\relax
  \providecommand{\doi}[1]{doi: #1}\else
  \providecommand{\doi}{doi: \begingroup \urlstyle{rm}\Url}\fi

\bibitem[Adler and Taylor(2009)]{adler2009random}
R.~Adler and J.~Taylor.
\newblock \emph{Random Fields and Geometry}.
\newblock Springer Monographs in Mathematics. Springer New York, 2009.
\newblock ISBN 9780387481166.
\newblock URL \url{https://books.google.de/books?id=R5BGvQ3ejloC}.

\bibitem[Ahn et~al.(2020)Ahn, Yun, and Sra]{ahn_sgd_2020}
K.~Ahn, C.~Yun, and S.~Sra.
\newblock {SGD} with shuffling: optimal rates without component convexity and
  large epoch requirements.
\newblock In \emph{Proceedings of the 34th {International} {Conference} on
  {Neural} {Information} {Processing} {Systems}}, {NIPS} '20, pages
  17526--17535, Red Hook, NY, USA, Dec. 2020. Curran Associates Inc.
\newblock ISBN 978-1-71382-954-6.

\bibitem[Ankirchner and Perko(2022)]{perko2022towards}
S.~Ankirchner and S.~Perko.
\newblock {Towards diffusion approximations for stochastic gradient descent
  without replacement}.
\newblock working paper or preprint, Jan. 2022.
\newblock URL \url{https://hal.science/hal-03527878}.

\bibitem[Ankirchner and Perko(2024)]{perko2024compare}
S.~Ankirchner and S.~Perko.
\newblock A comparison of continuous-time approximations to stochastic gradient
  descent.
\newblock \emph{Journal of Machine Learning Research}, 25\penalty0
  (13):\penalty0 1--55, 2024.
\newblock URL \url{http://jmlr.org/papers/v25/23-0237.html}.

\bibitem[Bingham et~al.(1987)Bingham, Goldie, and
  Teugels]{Bingham_Goldie_Teugels_1987}
N.~H. Bingham, C.~M. Goldie, and J.~L. Teugels.
\newblock \emph{Regular Variation}.
\newblock Encyclopedia of Mathematics and its Applications. Cambridge
  University Press, 1987.

\bibitem[Friz and Hairer(2020)]{friz2020course}
P.~Friz and M.~Hairer.
\newblock \emph{A Course on Rough Paths: With an Introduction to Regularity
  Structures}.
\newblock Universitext. Springer International Publishing, 2020.
\newblock ISBN 9783030415563.
\newblock URL \url{https://www.hairer.org/notes/RoughPaths.pdf}.

\bibitem[Friz and Victoir(2010)]{Friz_Victoir_2010}
P.~K. Friz and N.~B. Victoir.
\newblock \emph{Multidimensional Stochastic Processes as Rough Paths: Theory
  and Applications}.
\newblock Cambridge Studies in Advanced Mathematics. Cambridge University
  Press, 2010.

\bibitem[Gürbüzbalaban et~al.(2021)Gürbüzbalaban, Ozdaglar, and
  Parrilo]{gurbuzbalaban_why_2021}
M.~Gürbüzbalaban, A.~Ozdaglar, and P.~A. Parrilo.
\newblock Why random reshuffling beats stochastic gradient descent.
\newblock \emph{Mathematical Programming}, 186\penalty0 (1):\penalty0 49--84,
  Mar. 2021.
\newblock ISSN 1436-4646.
\newblock \doi{10.1007/s10107-019-01440-w}.
\newblock URL \url{https://doi.org/10.1007/s10107-019-01440-w}.

\bibitem[Jain et~al.(2020)Jain, Nagaraj, and Netrapalli]{jain_sgd_2020}
P.~Jain, D.~Nagaraj, and P.~Netrapalli.
\newblock {SGD} without {Replacement}: {Sharper} {Rates} for {General} {Smooth}
  {Convex} {Functions}, Feb. 2020.
\newblock URL \url{http://arxiv.org/abs/1903.01463}.
\newblock arXiv:1903.01463 [math].

\bibitem[Koren and Mansour()]{koren_benign_nodate}
T.~Koren and Y.~Mansour.
\newblock Benign {Underfitting} of {Stochastic} {Gradient} {Descent}.

\bibitem[Li et~al.(2017)Li, Tai, and Weinan]{Li15}
Q.~Li, C.~Tai, and E.~Weinan.
\newblock Stochastic modified equations and adaptive stochastic gradient
  algorithms.
\newblock In \emph{International Conference on Machine Learning}, pages
  2101--2110. PMLR, 2017.

\bibitem[Li et~al.(2019)Li, Tai, and Weinan]{Li18}
Q.~Li, C.~Tai, and E.~Weinan.
\newblock Stochastic modified equations and dynamics of stochastic gradient
  algorithms {I}: {Mathematical} foundations.
\newblock \emph{The Journal of Machine Learning Research}, 20\penalty0
  (1):\penalty0 1474--1520, 2019.

\bibitem[Li and Milzarek(2022)]{li_unified_2022}
X.~Li and A.~Milzarek.
\newblock A {Unified} {Convergence} {Theorem} for {Stochastic} {Optimization}
  {Methods}.
\newblock \emph{Advances in Neural Information Processing Systems},
  35:\penalty0 33107--33119, Dec. 2022.
\newblock URL
  \url{https://proceedings.neurips.cc/paper_files/paper/2022/hash/d630537fc4402cfa3ebbc7450a0cac91-Abstract-Conference.html}.

\bibitem[Mandt et~al.(2015)Mandt, Hoffman, and Blei]{mandt2015continuous}
S.~Mandt, M.~D. Hoffman, and D.~M. Blei.
\newblock Continuous-time limit of stochastic gradient descent revisited.
\newblock \emph{8th International Workshop on "Optimization for Machine
  Learning"}, 2015.

\bibitem[Mishchenko et~al.(2020)Mishchenko, Khaled, and
  Richtarik]{mishchenko_random_2020}
K.~Mishchenko, A.~Khaled, and P.~Richtarik.
\newblock Random {Reshuffling}: {Simple} {Analysis} with {Vast} {Improvements}.
\newblock In \emph{Advances in {Neural} {Information} {Processing} {Systems}},
  volume~33, pages 17309--17320. Curran Associates, Inc., 2020.
\newblock URL
  \url{https://proceedings.neurips.cc/paper/2020/hash/c8cc6e90ccbff44c9cee23611711cdc4-Abstract.html}.

\bibitem[Nagaraj et~al.(2019)Nagaraj, Jain, and Netrapalli]{nagaraj2019}
D.~Nagaraj, P.~Jain, and P.~Netrapalli.
\newblock Sgd without replacement: Sharper rates for general smooth convex
  functions.
\newblock 97:\penalty0 4703--4711, 09--15 Jun 2019.
\newblock URL \url{https://proceedings.mlr.press/v97/nagaraj19a.html}.

\bibitem[Nguyen et~al.()Nguyen, Mltd, Tran-Dinh, Phan, Nguyen, and van
  Dijk]{nguyen_unied_nodate}
L.~M. Nguyen, L.~Mltd, Q.~Tran-Dinh, D.~T. Phan, P.~H. Nguyen, and M.~van Dijk.
\newblock A {Uniﬁed} {Convergence} {Analysis} for {Shuﬄing}-{Type}
  {Gradient} {Methods}.

\bibitem[Perko(2025)]{perko_modified_2025}
S.~Perko.
\newblock Modified {Equations} for {Stochastic} {Optimization}, Nov. 2025.
\newblock URL \url{http://arxiv.org/abs/2511.20322}.
\newblock arXiv:2511.20322 [math].

\bibitem[Rajput et~al.(2020)Rajput, Gupta, and
  Papailiopoulos]{rajput_closing_2020}
S.~Rajput, A.~Gupta, and D.~Papailiopoulos.
\newblock Closing the convergence gap of {SGD} without replacement.
\newblock In \emph{Proceedings of the 37th {International} {Conference} on
  {Machine} {Learning}}, pages 7964--7973. PMLR, Nov. 2020.
\newblock URL \url{https://proceedings.mlr.press/v119/rajput20a.html}.
\newblock ISSN: 2640-3498.

\bibitem[Rajput et~al.(2021)Rajput, Lee, and
  Papailiopoulos]{rajput_permutation-based_2021}
S.~Rajput, K.~Lee, and D.~Papailiopoulos.
\newblock Permutation-{Based} {SGD}: {Is} {Random} {Optimal}?
\newblock Oct. 2021.
\newblock URL \url{https://openreview.net/forum?id=YiBa9HKTyXE}.

\bibitem[Renardy and Rogers(2006)]{renardy2006introduction}
M.~Renardy and R.~Rogers.
\newblock \emph{An Introduction to Partial Differential Equations}.
\newblock Texts in Applied Mathematics. Springer New York, 2006.
\newblock ISBN 9780387216874.
\newblock URL \url{https://books.google.de/books?id=IrIPBwAAQBAJ}.

\bibitem[Shamir(2016)]{shamir2016}
O.~Shamir.
\newblock Without-replacement sampling for stochastic gradient methods.
\newblock 29, 2016.
\newblock URL
  \url{https://proceedings.neurips.cc/paper/2016/file/c74d97b01eae257e44aa9d5bade97baf-Paper.pdf}.

\end{thebibliography}

\end{document}